\documentclass[conference,compsoc]{IEEEtran}

\usepackage{cite}
\usepackage{amsmath,amssymb,amsthm,amsfonts}
\usepackage{algorithmic}
\usepackage{algorithm}
\usepackage{amsmath}
\usepackage{amsthm}
\usepackage{graphicx}
\usepackage{textcomp}
\usepackage{xcolor}
\usepackage{booktabs}
\usepackage[hidelinks]{hyperref}
\usepackage{graphicx}
\usepackage{subcaption}
\usepackage{array}
\usepackage{booktabs}
\usepackage{longtable}
\usepackage{comment}

\newtheorem{lemma}{Lemma}
\newtheorem{theorem}{Theorem}
\newtheorem{corollary}{Corollary}
\begin{document}

\title{Trustformer: A \textbf{Trust}ed Federated Trans\textbf{former}}

% Submissions should be anonymized. See the CFP for details on how to anonymize your paper, including any references to your own work.
%\author{\em Anonymous Authors}

% The author information is skipped here, but can be used to include author information in the publication.

\author{\IEEEauthorblockN{1\textsuperscript{st} Ali Abbasi Tadi}
\IEEEauthorblockA{\textit{University of Windsor} \\
Windsor, Ontario, Canada \\
Abbasit@uwindsor.ca}
\and
\IEEEauthorblockN{2\textsuperscript{nd} Dima Alhadidi}
\IEEEauthorblockA{\textit{University of Windsor} \\
Windsor, Ontario, Canada \\
Dima.Alhadidi@uwindsor.ca}
\and
\IEEEauthorblockN{3\textsuperscript{rd} Luis Rueda}
\IEEEauthorblockA{\textit{University of Windsor} \\
Windsor, Ontario, Canada \\
Lrueda@uwindsor.ca}
%% IEEE format can accommodate up to six authors this way
}

\maketitle

\begin{abstract}
Transformers are deep-learning architectures primarily designed to process sequential data, achieving state-of-the-art performance across various tasks in sequence analysis. Models like BERT and GPT-3, derived from the Transformer architecture, have become foundational in solving broad challenges in Natural Language Processing (NLP) and have led to the development of large language models (LLMs). Despite the significant advancements made by LLMs, a critical issue remains: protecting the privacy of data owners whose information is used in LLM training. Privacy-preserving machine learning techniques, such as Federated Learning (FL), offer potential solutions for applying Transformers. However, several challenges limit the practical application of FL for Transformer training: (I) Recent research has shown that FL models can still leak sensitive information, primarily due to the FedAvg or FedSGD aggregation methods, and (II) the substantial size of Transformer models introduces a heavy communication burden when exchanging model weights between clients and a central server.
This paper proposes a novel FL method that reduces communication overhead while maintaining comparable utility to state-of-the-art methods. Our approach keeps model weights local by simulating a global model on the client side, thereby avoiding the need to share local weights with the central server. Instead, we apply $k$-means clustering on each layer of the Transformer model, identify centroids for each layer locally, and share only these centroids with the server in place of full model weights or gradients. To further enhance model security, we use Intel SGX to securely transmit centroids between clients and the server. We evaluated our scheme on a translation task and observed that it achieved a utility level on par with comparable baselines. The proposed method significantly reduces the data transferred between entities, resulting in a more efficient and privacy-preserving FL approach for Transformers.

\end{abstract} 

% Depending on how vigilant their paper processor is, IEEE may ask for these in your final paper, but we've heard about these amazing inventions called search engines that are able to index every word in your paper, so no need to include them in your submission unless you really want to.

%\begin{IEEEkeywords}
%component, formatting, style, styling, insert
%\end{IEEEkeywords}

\section{Introduction}

Federated Learning (FL),  developed by Google in 2017~\cite{mcmahan2017communication}, is a distributed Machine Learning (ML) approach that enables training models across multiple decentralized devices holding local data samples without exchanging them. This paradigm is especially beneficial for tasks that involve sensitive data, as it mitigates privacy risks by keeping the data localized. However, this scheme does not completely safeguard data from privacy attacks, since recent works have identified FL models vulnerable to some privacy attacks~\cite{gong2023gradient,dayal2023comparative}. % such as data reconstruction and model inversion attacks. %For instance, data reconstruction and model inversion attacks happen when the adversary has access to the model parameters being transmitted between clients and the server~\cite{gong2023gradient}. In addition, FL has unique challenges, such as problems with handling non-independent and identically distributed (Non-IID) data, communication overhead for large models, and the need for efficient model aggregation without accuracy loss~\cite{kairouz2021advances}. 
In general, two types of aggregation are widely used in FL: Federated Averaging (FedAvg) and Federated Stochastic Gradient Descent (FedSGD). In FedAvg, clients send their entire model (weights) to the central server, where the global model is created by averaging the clients' weights. In FedSGD, only the gradients are sent to the central server to construct the global model. Both aggregations are vulnerable to privacy attacks~\cite{gong2023gradient}, implying that sensitive information is retrievable from gradients or weights exchanged between the clients and the central server.

Transformers~\cite{vaswani2017attention,khan2022transformers} have revolutionized Natural Language Processing (NLP) tasks. They have been commonly used for tasks such as machine translation~\cite{raffel2020exploring}, text summarization~\cite{zhang2020pegasus}, and question answering~\cite{khashabi2020unifiedqa}. Transformers were initially applied to NLP tasks that require exploiting complex dependencies among elements in the input sequences (e.g., dependency parsing for determination of grammatical structures). Subsequently, they have been applied to Visual Question Answering (VQA)~\cite{zhang2021vinvl}, image classification~\cite{dosovitskiy2021image}, image generation~\cite{esser2021taming}, and protein fold prediction~\cite{jumper2021highly}. %More information about Transformer and its variants can be found in ~\cite{khan2022transformers}. %The Transformer that is trained on a large and general dataset (e.g., general text corpus) is called a pre-trained model (e.g., BERT, GPT-3) and is known to deal with various tasks. When applying a pre-trained model to a specific task, the model is fine-tuned by transfer learning using a small amount of training data specifically prepared for that task. Combining pre-training and transfer learning can achieve highly accurate recognition performance even with small amounts of training data in a reasonable computation time.
Integrating Transformers with FL promises to harness both technologies' strengths, potentially offering robust NLP models while preserving data privacy. However, FL cannot fully protect Transformers from privacy breaches as the federated mechanism has privacy vulnerabilities. A  recent research~\cite{fowl2022decepticons} has demonstrated that Transformers that apply FL for training language models are vulnerable to privacy breaches such as reconstruction attacks, which makes the central server able to reconstruct the clients' data from the weights being transferred from clients.

Several approaches in the literature have proposed mechanisms to protect the privacy of data owners during deep learning model training~\cite{abbasi2023comparative}. Techniques such as differential privacy~\cite{bu2022scalable}, multi-party computation~\cite{knott2021crypten}, and homomorphic encryption~\cite{lee2022privacy} have already been applied to standard neural network models, yielding reasonable results in simpler tasks and smaller models, such as image classification using convolutional neural networks. However, when it comes to complex models like Transformers, using these basic approaches cannot provide good results in large language models because of the complexity and the huge size of those models~\cite{chen2022fedtune}. %%With all the deficiencies of FL, it is still the best way for collaborative learning, providing guaranteed data privacy and reducing communication overhead. There is a research gap in this area, and the works that propose 
Applying FL to Transformers is often impractical due to utility loss and communication overhead. Recent studies have either caused a significant drop in utility when making FL Transformers resilient to privacy attacks~\cite{chen2022fedtune} or incurred substantial communication overhead when transferring large models between clients and the server~\cite{shi2023make}.
%One of the robust approaches for privacy-preserving machine learning is FL. In FL, multiple data owners train their local models and then share them with a remote server. The remote server is responsible for the aggregation of the local models, generates a global model, and sends it to the clients to do the second round of training. After some rounds the resulting model in all clients will be the global model as if the centralized model has been trained on all the data. 

Researchers have proposed various improvements to lower communication overhead and increase the accuracy of global models in FL~\cite{cheng2022differentially}. In this paper, we propose a new FL approach for neural networks that provides no privacy leakage. Additionally, our approach significantly reduces the model size transmitted between clients and the central server, while maintaining accuracy comparable to state-of-the-art methods. %We conducted comprehensive experiments to highlight the strengths and deficiencies of our approach for the research community.
We propose a secure FL framework (Trustformer) that can be applied to secure Transformer training from scratch. We applied our proposed method to a complex Transformer model for a translation task and observed high-quality translations with reduced communication overhead compared to state-of-the-art methods. To the best of our knowledge, Trustformer is the first approach that guarantees the privacy of the data owners and reduces network overhead while not dropping the accuracy of complex models like Transformers.  Trustformer is useful when training a federated Transformer from scratch and not fine-tuning a pre-trained model. The contribution of our work is summarized as follows.

\begin{itemize}
    \item  We propose a new federated learning approach, Trustformer, designed for large deep learning models. Instead of averaging the global model weights at the central server, we cluster the weights using  $k$-means and perform centroid averaging at the central server. The central server will not have access to the final global model at the end. Not having the final model on the server means there is less centralized information that could be exposed in the event of a security breach. 
    
    \item We show that since we transmit the centroids of weights, the new federated learning approach can significantly reduce the volume of data being transmitted between entities.
    \item We apply our proposed model for training a federated Transformer model from scratch. We used the translation task for this purpose. In addition, we provide comprehensive results confirming that our scheme works well in large Transformer models in terms of translation quality and model compression.
    \item We apply Intel SGX protection to protect the data being transmitted between various entities to enhance security.
    \item  We compare our approach with state-of-the-art methods (FedAvg~\cite{mcmahan2017communication}, DP-FedAvg~\cite{mcmahan2018learning}, DP-FedSAM~\cite{chen2022fedtune}, DP-BLUR-LUS~\cite{shi2023make}) for federated learning in terms of model size, output quality (BLEU score~\cite{post2018call}, METEOR~\cite{denkowski2014meteor}, and BERT F1~\cite{zhang2019bertscore}), demonstrating that our method performs well across all four metrics. 
    
\end{itemize}

The remainder of the paper is structured as follows. Section~\ref{sec:background} offers the essential background information for a thorough understanding of the paper. Section~\ref{sec:relatedwork}
delves into the related work. The system architecture followed by the threat model and assumptions are discussed in detail in Section~\ref{sec:systemoverview}. The proposed method is detailed in Section~\ref{sec:overview}. Mathematical analysis is provided in Section~\ref{sec:analysis}.
Section~\ref{sec:experiments} presents the experimental results, whereas Section~\ref{sec:discussion} provides a thorough security analysis. Finally, Section~\ref{sec:conclusion} concludes the paper, offering insights into future directions for research.

\section{Background}
\label{sec:background}
In this section, we present some preliminaries necessary for understanding the proposed framework.

\subsection{Transformer Architecture}
Transformers are a revolutionary type of neural network architecture designed to handle sequential data, especially when dealing with Natural Language Processing (NLP)~\cite{vaswani2017attention}. Transformers have set new benchmarks in various NLP tasks due to their ability to model long-range dependencies, parallelize training processes effectively, and require less time to train. The key components of the Transformer Architecture are: I) Encoder-Decoder Structure: A very simple Transformer model should include an encoder and a decoder. The encoder processes the input sequence to produce a context-rich representation whereas the decoder generates the output sequence from this representation, II) Self-Attention Mechanism: This is the core innovation of the Transformer. Self-attention is an attention mechanism relating different positions of a single sequence to compute a representation of the sequence.  %It allows each position in the input sequence to attend to all other positions, capturing relationships irrespective of their distance. 
This mechanism is implemented using multi-head attention, which enables the model to focus on different parts of the input sequence simultaneously, III) Feed-Forward Neural Networks: Both the encoder and decoder layers include position-wise fully connected feed-forward networks that apply non-linear transformations to the input, and IV) Positional Encoding: Since Transformers do not have a built-in notion of sequence order, positional encodings are added to the input embeddings to provide information about the positions of tokens in the sequence. Every architecture of deep neural networks, including Transformers, can be represented as a series of two-dimensional matrices. Transformers, like all deep neural network architectures, rely heavily on matrix operations for their computations. This matrix-based approach enables efficient implementation and optimization on modern hardware, facilitating the handling of complex and large-scale data across various applications.

\subsection{$k$-Means Clustering}
$k$-means clustering is a widely used unsupervised machine learning algorithm for partitioning a dataset into $k$ distinct, non-overlapping clusters~\cite{miao2023k}. Its goal is to group similar data points together while maximizing the distance among clusters~\cite{tadi2022nicasn}. Transformers and other neural network architectures can be represented using multiple matrices, such as weight matrices, embedding matrices, and attention mechanism matrices. Applying $k$-means clustering to these matrices is the main idea of this paper, where we can reveal patterns and insights into the model's behaviour and structure. Clustering weights can be used for model pruning by identifying and removing redundant or less important weights, leading to a more compact and efficient model. $k$-means clustering offers a powerful tool for analyzing and understanding the complex structures of neural network models, including Transformers. By applying clustering techniques to the matrices representing model weights, embeddings, and attention mechanisms, valuable insights into the model's inner workings can be gained. This approach can optimize the model's structure, enhance its interpretability, and lead to more efficient, compact, and interpretable models, ultimately improving their performance and applicability in various tasks.

\subsection{Intel Software Guard Extension}
Intel Software Guard eXtensions (SGX) is a set of security-related instruction code that are built into modern Intel processors. SGX is designed to provide a secure environment for sensitive computations by isolating specific code and data in memory regions called enclaves~\cite{tadi2024pppct}. These enclaves are protected from external software attacks, even if the operating system or BIOS is compromised. Enclaves are secure memory regions where code and data can be executed and stored. Only code running inside the enclave can access the data within it, providing strong isolation from the rest of the system. The key concepts related to SGX are: I) Memory Encryption: Data within enclaves are encrypted in memory, preventing unauthorized access or tampering by external processes or hardware attacks, II) Attestation: SGX provides mechanisms for remote attestation, allowing software running on a remote server to verify that it is running inside a genuine SGX enclave and has not been tampered with, and  III) Sealing: Data can be sealed to an enclave, allowing it to be securely stored outside the enclave and later retrieved and decrypted only by the same enclave, ensuring data integrity and confidentiality.

\section{Related Work}
\label{sec:relatedwork}
Table~\ref{tab:relatedwork} summarized and compared the related work from different perspectives. The FL quality for the translation task is measured by three score metrics such as BLEU, METEOR, and BERT F1, explained in section~\ref{sec:experiments}. In the following, we will detail those approaches based on the privacy perspective.

\begin{table*}[h]
    \centering
    \caption{Summary of related work}
    \scalebox{0.8}{
    \begin{tabular}{lcccll}
        \toprule
        \textbf{Aggregation Method} & \textbf{Privacy} & \textbf{FL Quality} & \textbf{Network Overhead} & \textbf{Supporting models} & \textbf{Privacy Mechanism} \\
        \midrule
        \textbf{FedAvg} \cite{mcmahan2017communication} & No & High & High & Simple/Complex & N/A \\
        \textbf{FedSGD} \cite{mcmahan2017communication} & No & High & Low & Simple & N/A \\
        \textbf{FedMA} \cite{wang2020federated}  & No & High & Low & Simple/Complex & N/A \\
        \textbf{DP-FedAvg} \cite{mcmahan2018learning} & Yes & Low & High & Simple/Complex & Gradient clipping + DP \\
        \textbf{FedDPA}~\cite{NEURIPS2023_e4724af0} & Yes & High & High & Simple & Fisher Matrix + DP \\
        \textbf{DP-FedSAM}~\cite{shi2023make} & Yes & High & High & Simple & Gradient clipping + DP\\
        \textbf{BLUR+LUS}~\cite{cheng2022differentially} & Yes & Medium & Medium & Simple/Complex & Gradient clipping + DP \\
        \textbf{Trustformer ($\beta=0.9$)} (Ours) & Yes & High & Medium & Simple/Complex & Generalization + TEE \\
        \bottomrule
        \label{tab:relatedwork}

    \end{tabular}}
\end{table*}
\subsection{Non-private Methods}
Here, we look at related work to see if they are aggregating the updates privately or not. Federated Averaging (FedAvg)~\cite{mcmahan2017communication} is a fundamental aggregation method in FL where model parameters from different clients are averaged to create a global model. This approach helps efficiently aggregate updates from multiple federated clients. Although FedAvg is an old method, it is still applicable to Transformers. However, there are significant concerns with it: I) The privacy of data owners (clients) is at risk if attackers gain access to the weights transmitted from the client to the central server, and II) Since this approach averages whole model parameters, it can apply to Transformer models but provides huge communication overhead because of the big size of the Transformer models that may contain millions of parameters. Federated Stochastic Gradient Descent (FedSGD)~\cite{mcmahan2017communication}, and Federated Matching Averaging (FedMA)~\cite{wang2020federated} improve FedAvg in terms of network overhead and supporting models, respectively. %We included the basic non-private methods in our work to evaluate how our privacy preserving method will cost us in terms of model utility. 
Other methods proposed in the literature that do not consider privacy include FedDyn~\cite{acar2021federated}, which dynamically updates regularizers on each device to align local and global model objectives, improving communication efficiency and robustness across diverse data distributions and device conditions. FedDyn can address the issue of different model updates across layers in deep architectures like Transformers, ensuring consistent updates across all layers. Attention mechanism-based aggregation~\cite{ji2021learning} leverages the attention mechanism within Transformers to assign different weights to updates that are coming from other clients based on their effect on the global model. This approach can enhance the aggregation process by prioritizing more significant updates. Adaptive Federated Optimizer (FedOpt)~\cite{reddi2021adaptive} uses adaptive optimization techniques like Adam or SGD with momentum for aggregating model updates. This approach can help achieve better convergence and performance for large models like Transformers. Many approaches~\cite{mcmahan2017communication, wang2020federated, acar2021federated, ji2021learning, reddi2021adaptive} focus on aggregation improvement of FL when applying complex models but they  do not consider FL privacy threats.

\subsection{Private Methods}
The remainder of the related work focuses on privacy-preserving methods proposed in the literature. The private version of FedAvg is Differential Privacy FedAvg (DP-FedAvg)~\cite{mcmahan2018learning}. FL with secure aggregation and Differential Privacy (DP)~\cite{phong2018privacy, truex2019hybrid, li2020convergence} combines secure aggregation with differential privacy to protect individual updates using cryptographic techniques and by adding noise to preserve privacy. Each client's data is perturbed before being encrypted and sent to the server for secure aggregation. This approach, however, causes a high accuracy drop due to the use of DP. DP-SGD~\cite{agarwal2018cpsgd} is a private version of FedSGD to protect individual data points while training simple machine learning models. However, since the optimization method in Transformers is different from SGD (typically Adam in LLMs), this approach does not apply to federated Transformers as it would lead to gradient explosion. Privacy-preserving FL with Trustworthy Aggregation~\cite{kang2020incentive} focuses on developing trustworthy aggregation methods that ensure data privacy and integrity, often using blockchain technology to enhance the trustworthiness of the aggregation process. Incorporating blockchain technology into the aggregation process can introduce significant computational and communication overhead.  Model distillation for FL~\cite{lin2020deep} involves each client training a smaller model (student) on its data, which are then aggregated into a global model (teacher). This method can reduce communication overhead and improve training efficiency. However, this approach is limited to fine-tuning Transformers and cannot be used for training a Transformer from scratch. FL with Transformer-based aggregation~\cite{ahuja2020federated} employs Transformers directly within the aggregation process, utilizing the model's capability to handle sequential data and dependencies to improve the aggregation of updates from various clients. Decentralized Aggregation~\cite{lalitha2019fully} in FL allows model updates to be shared and aggregated among clients without a central server. This approach can improve privacy and reduce the risk of a single point of failure. 

Shi \emph{et  al.}~\cite{shi2023make} proposed a novel approach called DP-FedSAM to address performance degradation in Differentially Private FL (DPFL). Traditional DPFL methods suffer from a sharp loss landscape and poor weight perturbation robustness, leading to diminished model performance. DP-FedSAM incorporates Sharpness-Aware Minimization (SAM) to generate locally flat models that mitigate these issues. By stabilizing local updates and reducing the negative impact of DP noise, DP-FedSAM improves both the generalization ability and robustness of the model.  Shi \emph{et  al.}~\cite{shi2023make} offers a theoretical analysis using Rényi Differential Privacy (DP) for privacy guarantees and empirically shows that DP-FedSAM achieves state-of-the-art performance on datasets like EMNIST and CIFAR-10 compared to existing DPFL baselines. Cheng \emph{et al.}~\cite{cheng2022differentially} introduced two key techniques: Bounded Local Update Regularization (BLUR) and Local Update Sparsification (LUS) to improve the utility of models in user-level DPFL. BLUR limits the norm of local updates before the clipping operation, reducing the negative impact of noise addition on the final model. LUS further enhances this by sparsifying local updates, keeping only the most relevant parameters and zeroing out the rest, effectively reducing the magnitude of updates while maintaining model accuracy. Cheng \emph{et al.}~\cite{cheng2022differentially} provided a theoretical analysis of convergence and privacy guarantees, and demonstrated through experiments that their method significantly improves the privacy-utility tradeoff in federated learning with differential privacy guarantees compared to existing approaches.

\section{System Overview}\label{sec:systemoverview}
In this section, we detail Trustformer's architecture, followed by the threat model and assumptions.

%\section{Proposed method}
%\label{sec:proposed}
%Here, we propose our Trustformer method, which is capable of running FL for both simple and complex models.
\begin{figure}[htbp]
  \centering
  \includegraphics[width=.4\textwidth]{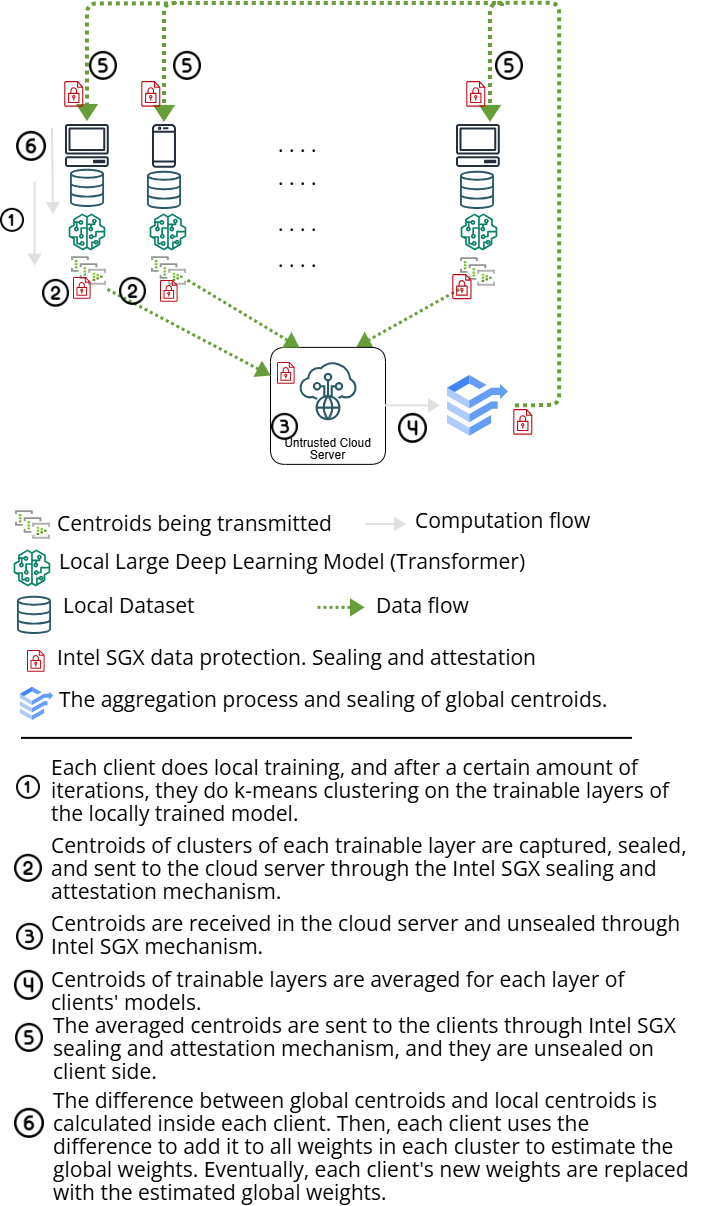} % Replace 'example-image' with the filename of your image
  \caption{Architecture}
  \label{fig:architecture}
\end{figure}
\subsection{Architecture}
Following Figure~\ref{fig:architecture}, several clients jointly train a Transformer from scratch based on their own dataset. We also have a central server in the  Trustformer's architecture responsible for the aggregation of the clients' local parameters. Each client has an Intel SGX processor and can communicate securely with the central server. Each client trains its local model and sends some information about its local model to the server for aggregation. Once aggregation is done on the server, the results are sent back to the clients. Clients use those results to estimate the global model. This process takes place over one (or multiple) iteration until all clients reach the required performance. To further enhance the privacy and security aspects of federated learning of Transformers, we incorporate Intel SGX, a hardware-based Trusted Execution Environment (TEE), into our federated learning framework. Intel SGX provides a secure enclave for computations, ensuring that the data being transmitted over the network are protected even on potentially compromised devices. This addition is depicted in Figure~\ref{eq:aggregation}, where each SGX client operates within a secure enclave, enabling encrypted communication of local and global parameters.

\subsection{Threat Model}
\label{sec:threat}
We consider two threat models. First, we assume that the clients and the central server are honest-but-curious. They act honestly and correctly following the designated framework specification. However, any of these entities may intentionally and curiously try to break privacy by inspecting the data received from other entities. The second threat model is that the clients are honest-but-curious but the server is compromised and not trusted. The central server will use the client's updates to reconstruct the client's data. We show that Trustformer safeguards the model from all types of attacks.

\subsection{Assumptions}
Each client $i$ has a subset of the dataset $D$ that contains sequence data records $D=\{(x_1, y_1),(x_2, y_2),..., (x_{|D|},y_{|D|})\}$ where $x_m$ is the data sequence $m$, and $y_m$ is the target sequence of the source $x_m$. In our translation task, we consider $x_m$ as the original sequence (like a sentence in the Russian language) and $y_m$ as the translation of the $ x_m$ into English. Additionally, each client is responsible for training its local model using ${\frac{|D|}{n}}$ data records, where $n$ is the number of clients. Then, all clients should use the proposed aggregation method to converge to the global model. We assume each Transformer model in clients consists of a stack of $l$ trainable layers $ M^{local} =\{m_1,m_2, ..., m_l\}$ where $m_j$ for $j=1,2,\dots,l$ is a trainable layer in the Transformer model $M^{local}$.  Transformers usually include non-trainable layers as well, like dropout or normalization layers, that are not included in $M^{local}$ since their values are not trainable. Each $m_j$ produces a matrix $w_j$, which shows its optimal weights (or biases) after training. In addition, we assume each weight matrix ($w_j^{r\times f}$) has the dimension of $r$ records and $f$ features. For example, the positional encoding in our experiments, which uses a Cross-lingual Language Model (XLM)~\cite{lample2019cross}  tokenizer, has $r=250000$ and  $f=256$. For simplicity, we consider $r$ equals the number of $m_{j-1}$'s layer neurons, and $f$ equals the number of neurons in the $m_j$ layer. Thus, each $w_j$ is a matrix, with dimensions  $r\times f$. The content of the matrix is the weights (or biases) of neurons. We propose to run a $k$-means clustering on each $w_j$ of each client's local model $M^{local}$ after local training. This clustering aims to obtain the centroids points $c^i_j = \{\hat{c}_1, \hat{c}_2, ..., \hat{c}_{No_c}\}$ where $c^i_j$ represents the set of centroids in layer $j$ of the model of client $i$, $No_c$ is the number of clusters, and $\hat{c}_k$, for $k=1,2,\dots,No_c$, is a vector of  size  $f$. Because the layer dimensions vary, our proposed algorithm adaptively adjusts $No_c$  rather than keeping it constant across all layers. In addition, we assume that all clients agree on the clustering scheme, including the clustering algorithm, clustering ratio ($\beta$) for each layer, and the same seed for generating similar random values. Since we are adaptively adjusting the number of clusters, the value of the clustering ratio ($\beta$) is common between clients.

\section{Trustformer}
\label{sec:overview}
In this section, we provide the details of the proposed method through various subsections.
\subsection{Overview}

As illustrated in Figure~\ref{fig:architecture}, in the first step, each client $i$ starts by training its local model ($M_{\text{local}}^i$). This local model is a Transformer model, which includes attention mechanisms to understand the text sequences better. Therefore, there are a significant number of trainable local parameters. Once the $i$-th client trains its local model (out of $n$ clients), it will result in the locally optimal weights for each model layer. Assuming all clients perform their local training in parallel, we will have a set of all locally trained models from all clients. Each client $i$ trains its local model $M_{\text{local}}^i$ to obtain the locally optimal weights: 

\begin{equation}
W_{local}^i = \text{Train}(M_{\text{local}}^i, D_i)
\end{equation}
\noindent where $\text{Train}(\cdot)$ represents the training process, including forward and backward passes and updates to the model parameters using the Adam optimizer~\cite{kingma2014adam}. $D_i$ is the dataset of client $i$ whose model should be trained on. Since the local model for client $i$ has $l$ trainable layers $M_{local}^i =\{m_1^i,m_2^i, ..., m_l^i\}$, the local optimal weights for each model is $W_{local}^i=\{w_1^i,w_2^i, ..., w_l^i\}$.
\begin{figure}[htbp]
  \centering
  \includegraphics[width=0.5\textwidth]{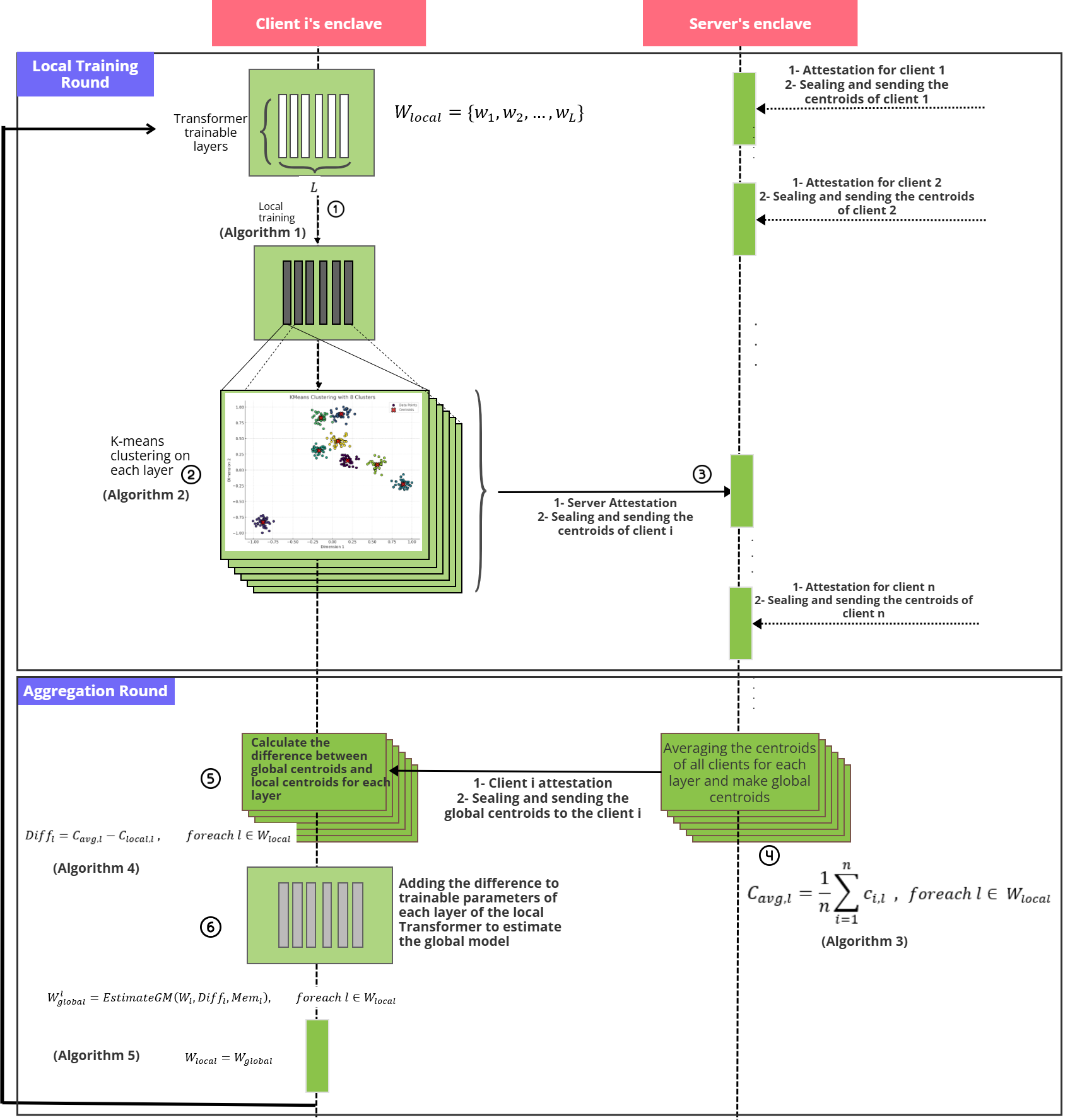} % Replace 'example-image' with the filename of your image
  \caption{Training and aggregation}
  \label{fig:aggregation}
\end{figure}
In the second step of Figure~\ref{fig:architecture}, we do $k$-means clustering on each layers's weights matrix of the Transformer. Then, we send the obtained centroids to the central server for aggregation. We also keep the labels of the clusters local to simulate the global model after the server sends back the aggregation results. In practice, the model weights $W_{local}^i$ is reduced to a set of centroids for all layers $C_{local}^i=\{c_1^i,c_2^i,...,c_l^i\}$, where the shape of $c_j^i \quad, \forall 1\leq j \leq l$ is $(No_c\times f)$. 
 
 Each centroid leads the coordinates of all data points in one cluster, so the number of clusters equals the number of centroids. Note that the number of clusters ($No_c$) is adaptively adjusted and way lower than the number of weights. Because of this reduction, we ensure the data transmitted between the client and server are way less than the whole model weights $W_{local}$. According to Theorem~\ref{theo:centroids}, clients eventually converge to the global model generated by FedAvg once the server averages the centroids instead of actual weights, provided that the number of clusters is sufficiently high. Other than mathematical proof on the Theorem~\ref{theo:centroids} shown in the Appendix, in Section~\ref{subsec:loss}, we experimentally prove Theorem~\ref{theo:centroids} by observing the loss function over different numbers of clusters (by changing clustering ratio $\beta$).

After completing local clustering, we leverage Intel SGX to send the $C_{local}^i$ to the central server. Although $C_{local}^i$, due to generalization, does not reveal much information about the weights in the model, for more robust protection, we attest the central server, seal $C_{local}^i$,  and then send it to the central server. Centroids are the averages of a bunch of weights.  If the number of centroids increases in our approach, each centroid represents fewer weights. For example, when the number of centroids matches the total number of weights, each centroid corresponds to a single weight, making our method behave like the baseline (FedAvg). As the number of centroids decreases, each centroid represents more weights, revealing less information about individual data. In step 3 in Figure~\ref{fig:architecture}, $C_{local}^i$ for all clients are received, unsealed, and aggregated through vector averaging as illustrated in Figure~\ref{fig:aggregation}. The aggregation happens in step 4 of Figure~\ref{fig:architecture}, by using Equation (\ref{eq:aggregation}). The central server layer-wise averages all centroids coming from clients to find global centroids ($C_{global} = \{c_{avg,1}, c_{avg,2}, ..., c_{avg,l}\}$) and send them back to clients to estimate the global model weights ($W_{global}$). %Equation~(\ref{eq:aggregation} shows how we do layer-wise averaging of centroids. 

\begin{equation}
\label{eq:aggregation}
c_{avg,j} = \frac{1}{n} \sum_{i=1}^{n}c_{i,j} , \quad \forall 1\leq j\leq l
\end{equation}

%\begin{theorem}[Averaging Centroids of the weights and FedAvg]
%\label{theo:centroids}

%Let $W_i$, $i = 1, 2, \ldots, n$, be $n$ matrices, each with $r$ rows and $f$ features. Performing $k$-means clustering on each matrix $W_i$ with $No_{c}$ clusters results in centroids $c_i = \{\hat{c}_{i1}, \hat{c}_{i2}, \ldots, \hat{c}_{iNo_c}\}$. Compute the average centroids across the matrices:
%\[
%c_{avg} = \{\frac{1}{n} \sum_{i=1}^{n} \hat{c}_{ij} | %\quad \forall 1 \leq j\leq No_c\}
%\]

%Then, shift all $r$ data points in each cluster of each matrix according to the difference vectors ($c_{avg} - c_{i}$) between the averaged centroids and the original centroids. As $No_{c} \rightarrow r$, where $r$ is the number of records in  $W^{r\times f}$, this process approximates the results obtained by the Federated Averaging (FedAvg) algorithm.
%\end{theorem}
%\begin{proof}
%\noindent The proof is in the Appendix.
%\end{proof}

The details of the aggregation mechanism are depicted in Figure~\ref{fig:aggregation}. 
%As per this figure, once all local centroids are delivered to the server , the aggregation process happens for each layer, and the results are sent back to the clients. 
 The steps of computations in Figure~\ref{eq:aggregation} are shown in circles. In the first step, the Transformer in client $i$ is locally trained, and the weights are set up during the first round of local training. In the second step, we apply $k$-means clustering on each of the weights of each layer and find the centroids of the weights of each layer of the local transformer ($C_{local}$). In Figure~\ref{fig:aggregation}, we assumed the layers have just two neurons (equal to two features), making an intuitive plotting of centroids for simplicity. However, in actual implementation, the layers can have an arbitrary number of neurons. In the third step, to protect privacy, the centroids are sealed and sent to the central server after attestation. In the fourth step, the central server unseals $C_{local}$, averages the local centroids layer-wise of all clients, and generates the global centroids ($C_{global}$). Then, the global centroids are sealed and sent back to all clients. Once client $i$ receives $C_{global}$, in the fifth step, the difference vector ($Diff$) between  $C_{global}$ and $C_{local}$ are computed by the client, and then, in step six, all the weights in $W_{local}$ are shifted according to $Diff$. Using this way, clients approximate the $W_{global}$ without having to share the individual weights ($W_{local}$) with the central server. A simple example of the approximation process is shown in Figure~\ref{fig:simulateGM} and explained in subsection~\ref{sec:proposedalg}. After the $W_{global}$ is approximated, the client replaces its local weights with the new approximated weights and starts the next round of training with the new weights.

\subsection{Proposed Algorithm}
\label{sec:proposedalg}
Algorithm~\ref{alg:trustformer} illustrates the overall process of our proposed approach Trustformer. We defined the aggregation frequency $Agg_{fr}$, which defines how often the aggregation should happen. $Agg_{fr}$ is different than the number of iterations ($R$). $R$ defines how many training iterations we should have in total. In some cases, we might need to do multiple iterations of local training without aggregation. We included $Agg_{fr}$ for the generalizability of the work. However, in our experiments, $Agg_{fr}$ is set to 1, meaning for each iteration, we do one aggregation with the server. For instance, if $Agg_{fr}=5$, in every 5 iterations, we do aggregation with the server. As in real-world FL scenarios, we do not need to apply aggregation on every single iteration~\cite{guha2019one}. Setting higher values for $Agg_{fr}$ helps reduce the training time in real-world use cases. Algorithm~\ref{alg:trustformer}  ultimately provides the global model for all clients without having clients share all their local parameters, in contrast to FedAvg. The inner loop of Algorithm~\ref{alg:trustformer} underscores the tasks that should be done on each client side. The $DoClustering$ algorithm is presented in Algorithm~\ref{alg:apply-clustering}, where we apply $k$-means clustering to each layer of $W_{local}$, resulting in the centroids ($C_{local}$) and the membership ($mem_{local}$) of the data points. The membership indicates which cluster each data point belongs to. For example, if $mem[2] = 3$, it means that data point 2 is a member of cluster 3. An example of the clustering result from Algorithm~\ref{alg:apply-clustering} is shown in Figure~\ref{fig:simulateGM}-a, where, for simplicity, we assume a single trainable layer with only two neurons, and the trainable weights are represented as f1 and f2. Additionally, there are 300 inputs connected to these two neurons (i.e. $r=300, f=2$), resulting in a weight matrix that is a matrix with 300 rows and two columns ($f_1$ and $f_2$). We perform clustering on the local model weights to separate the weights with similar behaviour in the model. Figure~\ref{fig:simulateGM}-a illustrates the clustering output on one client for a simple model with a single layer and 3 clusters ($No_c=3$). In the proposed algorithm Trustformer, the number of neurons is significantly higher than in the example shown in Figure~\ref{fig:simulateGM}. However,
the number of clusters ($No_c$) for Algorithm~\ref{alg:apply-clustering} is adaptively adjusted by Equation~(\ref{eq:no-c}) according to the number of tensors (rows) that they handle. In Equation~(\ref{eq:no-c}), we identified a clustering ratio ($\beta$) that can vary in the range $(0,1]$ and $r$, which is the number of all tensors in matrix $w_j$ for $j=1,2,\dots,l$. 

Once clients are done with clustering, they send the obtained centroids of the local model to the central server for aggregation. The aggregation happens in the server and provides the clients with the global centroids ($C_{global}$). The global centroids are simply the layer-wise average of all local centroids that are being sent to the server. Following Equation~(\ref{eq:aggregation}), the server does the averaging of the centroids. The aggregation algorithm is shown in Algorithm~\ref{alg:server-aggregation}, where the central server obtains the local centroids of all clients, performs the aggregation, and sends the global centroids to all clients. In Figure~\ref{fig:simulateGM}-b, we assumed that the server has already computed the global centroids (shown by triangles) and sent the global centroids to all clients to approximate the global model. 

In the next step (line 14 of Algorithm~\ref{alg:trustformer}), once the clients obtain the global centroids, they compute the difference vectors ($Diff$) between local and global centroids using Algorithm~\ref{alg:compute-difference}. The $Diff$ records the changes in centroids and subsequently determines how much and in what directions the local weight in each cluster in matrix $w_j$ should be shifted. For example,  Figure~\ref{fig:simulateGM}-b shows the difference vector in terms of $\Delta f_1$ (variation of $f_1$ feature) and $\Delta f_2$ (variation of $f_2$ feature). As it is clear in Figure~\ref{fig:simulateGM}, the number of elements in the $Diff$ set equals the number of layers in the model ($Diff=\{Diff_1\}$), and the number of points in each element of $Diff_1$ equals the number of clusters of each layer. Since in Figure~\ref{fig:simulateGM}-b, there is only 1 layer, ($Diff_1=\{(\Delta f_1, \Delta f_2)_{\hat{c}_1}, (\Delta f_1, \Delta f_2)_{\hat{c}_2}, (\Delta f_1, \Delta f_2)_{\hat{c}_3}\}$), where $\hat{c}_j$ shows the centroid of cluster $j$. $S^i_{l,j}$ represents the data points of layer $l$ of client $i$ in cluster $j$. If we consider the blue data points as cluster $S^i_{1,1}$ with center $\hat{c}_1$, yellow as cluster $S^i_{1,2}$ with center $\hat{c}_2$, and purple as cluster $S^i_{1,3}$ with center $\hat{c}_3$, $Diff_1=\{(-4.49, -3.01)_{\hat{c}_1}, (1.07, 2.49)_{\hat{c}_2}, (1.67, -3.85)_{\hat{c}_3}\}$.

In the final step, the clients need to approximate the global model by only having the global centroids and their local models. This happens in Algorithm~\ref{alg:estimate-gm}, and it provides the global model for Algorithm~\ref{alg:trustformer}. Algorithm~\ref{alg:estimate-gm} is responsible for shifting the data points to their correct location by looking at the centroid variations in $Diff$. A simple example is shown in Figure~\ref{fig:simulateGM}-c, where the data points in cluster $j$ ($S^i_{l,j}$) of the local model are shifted based on the deviation of the local cluster $j$'s centroid from the global centroids ($ (\Delta f_1, \Delta f_2)_{\hat{c}_j}$). Algorithm \ref{alg:estimate-gm} is the most important algorithm in our work since it helps clients approximate the global model weights with neither having access to the other clients' local model weights nor having server access to the clients' local model weights. Once clients estimate the global weights, they simply copy these weights to their local model so they can proceed with the next round of training. 

As we already mentioned, the server-side algorithm for generating the global centroids ($C_{global}$) is illustrated in Algorithm~\ref{alg:server-aggregation}. All centroids exchanged between clients and the server are encapsulated within Intel SGX enclaves to enhance security. This process is outlined in line 10, line 11, and line 13 of Algorithm~\ref{alg:trustformer}.  Clients perform attestation and sealing/unsealing when communicating with the central server. Line 5 and Line 14 of Algorithm~\ref{alg:server-aggregation} handle the sealing of centroids during transmission. Apart from Algorithm~\ref{alg:server-aggregation}, all other algorithms are executed on the client side, ensuring their privacy is already preserved.

\subsection{Cluster Size Adaptation}
While a fixed cluster size ($No_c$) is effective in our work for simple neural network models, we propose using an adaptive cluster size based on the dimensions of the weight matrix. To achieve this, we calculate $No_c$ using  Equation~(\ref{eq:no-c}). 
\begin{equation}
\label{eq:no-c}  
No_c = \lfloor{r}.\beta\rfloor
\end{equation}
\noindent where, $r$ is the number of rows in the weight matrix, and $\beta$ is the clustering ratio in our proposed method. The rationale for this adaptive selection is that different layers of the Transformer model have widely varying dimensions depending on their configuration.  This adaptive selection ensures that centroids of a fixed portion of the layers are considered for aggregation. For example, if we select $\beta = 0.2$, we will apply clustering to the positional encoding with $No_c = 0.2\times 250000=50000$. For feed-forward neural network layers, clustering will result in $No_c=\lfloor0.2\times 256\rfloor =51$. The $No_c$ ranges between $1$ and $r$.

\begin{figure*}[htbp]
  \centering
  \includegraphics[width=0.8\textwidth]{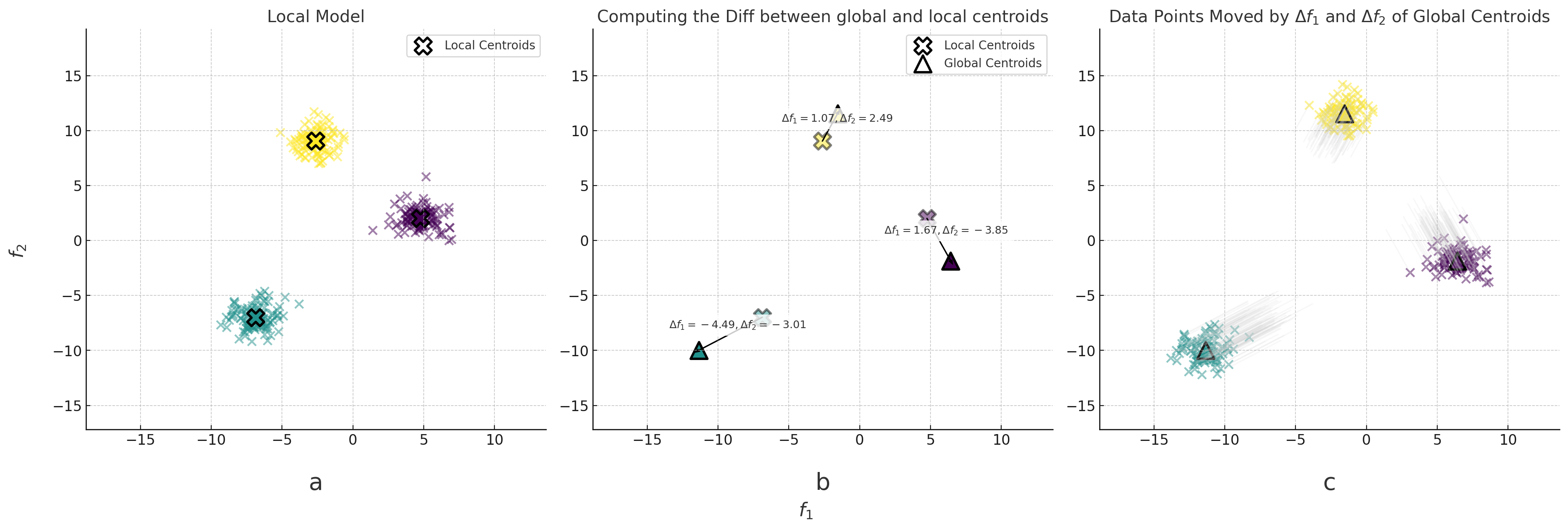} % Replace 'example-image' with the filename of your image
  \caption{Example of Simulating the global model by using the difference between global centroids and local centroids. For simplicity, We assumed there is only one trainable layer with two neurons and 300 inputs. The weights of the neurons are shown as $f_1$ and $f_2$. Since there are 300 inputs, 2 trainable weights and only one layer, we have a matrix of 300 rows and 2 columns as the model. a) Clustered local model with $No_c=3$. b) Distance between global centroids features and local centroids features are computed as set of differences ($Diff_1=\{(\Delta f_1, \Delta f_2)_{\hat{c}_1}, (\Delta f_1, \Delta f_2)_{\hat{c}_2}, (\Delta f_1, \Delta f_2)_{\hat{c}_3}\}$). c) Simulation of the global model by moving each data point according to the differences. $(\Delta f_1, \Delta f_2)_{\hat{c}_j}$ is added to each data point in cluster $S^i_{1,j}$.}
  \label{fig:simulateGM}
\end{figure*}

\begin{algorithm}
\scriptsize 
\caption{Trustformer}
\label{alg:trustformer}
\begin{algorithmic}[1]
\STATE \textbf{Input:} number of communication rounds $R$, number of clients $n$, aggregation Frequency $Agg_{fr}$, clustering ratio $\beta$, dataset $\mathcal{D}$  \\
\STATE \textbf{Output:} Trained global model on the whole dataset $\mathcal{D}$
\STATE Initialize local model $M_{local}^i$ for all clients in parallel where $1\leq i \leq n$
\STATE Load local dataset $\mathcal{D}_i$ for all clients in parallel
\FOR{each  $round = 1, 2, \ldots, R$}
    \FOR{each client $i = 1, 2, \ldots, n$ \textbf{in parallel}}
        \STATE $W_{local}^i = Train(M_{local}^i, D_i)$
        % \COMMENT{Train model on dataset $D_i$ to obtain optimized local model $M^{local}$}
        \IF{($round$ mod $Agg_{fr}$) = 0} %\COMMENT{Aggregation happens in the following}
            \STATE $C_{local}^i, Mem^i_{local}\gets DoClustering(W_{local}^i,\beta)$ \COMMENT{Using Algorithm~\ref{alg:apply-clustering}}
            \STATE Initiate remote attestation to connect to the server
            \STATE Seal and send local centroids updates ($C_{local}^{i}$) to the server
            \STATE $C_{global} \gets Aggregation({C^i_{local}})$ \COMMENT{Using Algorithm~\ref{alg:server-aggregation}}
            %\STATE Wait until the aggregation happens in the server and obtain sealed global centroids ($C_{global}$) from the server using Algorithm~\ref{alg:server-aggregation}
            \STATE unseal $C_{global}$
            \STATE $Diff^i \gets ComputeDifference(C_{global},C_{local}^i)$ \COMMENT{Using Algorithm~\ref{alg:compute-difference}}
            \STATE $W_{global}^i \gets EstimateGM(W_{local}^i, Diff^i, Mem^i_{local})$ \COMMENT{Using Algorithm~\ref{alg:estimate-gm}}
            \COMMENT{estimate the global model weights ($W_{global}$) by adding the $Diff^i$ to all the points in the same clusters using $Mem^i$ vector}
            \STATE Copy global weights $W_{global}$ to local model $M_{local}^i$
        \ENDIF
    \ENDFOR
\ENDFOR
\STATE \textbf{Return} model $M^{local}$
\end{algorithmic}
\end{algorithm}

\begin{algorithm}
\scriptsize 
\caption{DoClustering($W_{local}$,$\beta$)}
\label{alg:apply-clustering}
\begin{algorithmic}[1]
\STATE \textbf{Input:} Local Transformer model weights ($W_{local}$) , and clustering ratio $\beta$ 
\STATE \textbf{Output:} Set of centroids ($C_{local}$), and membership status of each record ($Mem_{local}$)
\STATE $C_{local} \gets \{\}$
\STATE $Mem_{local}\gets \{\}$
\FOR{$w_l^{r\times f}\in W_{local}$} %\COMMENT{each layer may have different values for $r$ and $l$}
    \STATE $No_c\gets\lfloor r.\beta \rfloor$ \COMMENT{each layer may have different values for $r$ and $f$}
    \STATE $C_l^{No_c \times f},mem_l^{r \times 1}\gets kmeans(w_l^{r\times f},No_c)$
    \STATE $C_{local}\gets C_{local} \cup C_l^{No_c \times f}$
    \STATE $Mem_{local}\gets Mem_{local}\cup mem_l^{r \times 1}$
\ENDFOR
\STATE \textbf{Return} $C_{local}$, $Mem_{local}$
\end{algorithmic}
\end{algorithm}

\begin{algorithm}
\scriptsize 
\caption{Server side Aggregation($C_{local}$)}
\label{alg:server-aggregation}
\begin{algorithmic}[1]
\STATE \textbf{Input:} Local centroids of $i$-th client ($C_{local}^{i}=\{c_1^i, c_2^i, ..., c_l^i\}$), where $l$ is the number of layers of weights. 
\STATE \textbf{Output:} Global centroids ($C_{global}$)
\STATE Wait until all centroids from $n$ clients are received
\STATE $C_{all}\gets \{C_{local}^{1}, C_{local}^{2}, ..., C_{local}^{n}\}$
\STATE Unseal $C_{all}$
\STATE $C_{global}\gets \{\}$
\FOR{$j \in \{1,2,.., l\}$} 
    \STATE $sum_j = 0$
    \FOR{$i \in \{1,2,.., n\}$}
        \STATE $sum_j \gets sum_j + c_j^{i}$ \COMMENT{$c_j^{i}$ is the local centroids of the $j$-th layer of the model weights in $i$-th client }
    \ENDFOR
    \STATE $C_{global} \gets C_{global} \cup (\frac{sum_j}{n})$
\ENDFOR
\STATE Seal $C_{global}$
\STATE \textbf{Return} Sealed $C_{global}$
\end{algorithmic}
\end{algorithm}

\begin{algorithm}
\scriptsize 
\caption{ComputeDifference($C_{global}$,$C_{local}$)}
\label{alg:compute-difference}
\begin{algorithmic}[1]
\STATE \textbf{Input:} Global centroids ($C_{global}$) , and local centroids ($C_{local}$) 
\STATE \textbf{Output:} Difference between global and local centroids ($Diff$)
\STATE $Diff\gets \{\}$
\FOR{$c_l^{No_c \times f} \in C_{local}$ and $c_g^{No_c \times f} \in C_{global}$} 
    % \STATE $No_c\gets\lfloor r.\beta \rfloor$
    \STATE $Diff_l^{No_c \times f}\gets (c_g^{No_c \times f} - c_l^{No_c \times f})$ \COMMENT{Element-wise subtraction}
    \STATE $Diff\gets Diff \cup Diff_l^{No_c \times f}$
\ENDFOR
\STATE \textbf{Return} $Diff$
\end{algorithmic}
\end{algorithm}

\begin{algorithm}
\scriptsize 
\caption{EstimateGM($W_{local}$,$Diff$,$Mem$)}
\label{alg:estimate-gm}
\begin{algorithmic}[1]
\STATE \textbf{Input:} Local model weights ($W_{local}$), Difference between global centroids and local centroids ($Diff$), membership status of each record in each trainable layer ($Mem$)  
\STATE \textbf{Output:} Approximated global model weights ($W_{global}$)
% \STATE $Diff\gets \{\}$
\STATE $W_{global}\gets \{\}$
\FOR{$w_l^{r \times f}\in W_{local}$, and $mem_l^{r \times 1}\in Mem$,  and $diff_l^{No_c \times f}\in Diff$}
    \STATE $temp_l \gets \{\}$
    \FOR{$i \in \{0,2,..,len(unique(mem_l))\} $}
        \STATE $temp\gets w_l^{r \times f}[mem_l^{r \times 1}[i]]+diff_l^{No_c \times f}[mem_l^{r \times 1}[i]]$
        \STATE $temp_l\gets temp_l\cup temp$
    \ENDFOR
    \STATE $W_{global}\gets W_{global} \cup temp_l$
    
\ENDFOR
\STATE \textbf{Return} $W_{global}$
\end{algorithmic}
\end{algorithm}

\section{Analysis}
\label{sec:analysis}
Theorem~\ref{theo:centroids} mathematically proves the convergence of the Trustformer with some error proportional to the selected number of clusters.

\begin{theorem}
 \label{theo:centroids}
 In a federated learning setting where each client clusters their local model parameters \( w_i \in \mathbb{R}^r \) into \( No_c \) clusters and updates their model by adjusting parameters based on the differences between global and local centroids, the clients' updated models \( w_i^{\text{global}} \) converge to the Federated Averaging (FedAvg) global model \( w_{\text{FedAvg}} = \frac{1}{n} \sum_{i=1}^{n} w_i \) as \( No_c \to r \).
\end{theorem}

\begin{proof}
\noindent The proof is provided in the Appendix.
\end{proof}

\noindent\emph{Convergence to FedAvg:}
The process ensures that the average of the clients' updated models equals the FedAvg global model. As 
$No_c \to r$, where $r$ is the number of parameters in each layer, each client's approximated model gets close to the FedAvg global model with some approximation error which is defined and analyzed in the Appendix.

\noindent\emph{Privacy and Communication Efficiency:}
Transmitting only $No_c$ centroids reduces communication overhead when $No_c<<r$. Clustering introduces quantization, which may enhance privacy by obscuring individual parameter details.

\noindent\emph{Trade-off:}

\noindent Smaller $No_c$:
\begin{enumerate}
    \item Greater communication efficiency and potential privacy benefits.
    \item Larger approximation error to FedAvg, causing utility loss
\end{enumerate}

\noindent Larger $No_c$:
\begin{enumerate}
    \item Better approximation to FedAvg, leading to better utility
    \item Increased communication cost.
\end{enumerate}

\section{Experimental Results}
\label{sec:experiments}
This section presents the experimental results of our Trustformer method applied to a translation task, enhanced with Intel Software Guard Extensions (SGX) for increased security and privacy. The primary objective of these experiments is to evaluate the performance and feasibility of applying the proposed federated learning to Transformer-based models for translation, particularly in a scenario where data privacy and security are paramount. The code and data are provided in our GitHub repository for replication~\cite{trustformergithub}. The experiments described here are designed to investigate the following key aspects:
\begin{itemize}

    \item We evaluate whether our approach converges to the global model to experimentally validate the Theorem~\ref{theo:centroids}. We demonstrate that by sufficiently increasing the number of clusters, the training loss of our model matches the loss of the baseline (FedAvg).
    \item We measure the training time of converging to the global model.
    \item We evaluate the quality of the approximated global model of the translation task in all clients separately. This quality is measured by multiple score metrics such as BLEU~~\cite{post2018call}, METEOR~\cite{denkowski2014meteor}, and BERT F1~\cite{zhang2019bertscore}. 
    \item We measure the effects of the number of clusters on the training time.
    \item We measure the communication overhead by observing the volume of data being sent to the network at each round of training by each client.
    % \item As the ablation studies, we set various aggregation frequencies ($Agg_fr$) and evaluate our work performance.
\end{itemize}

We compared the performance of our approach against recent privacy-preserving Transformer contributions and the baseline (FedAvg)~\cite{mcmahan2017communication}. We benchmarked our approach against  DP-FedSAM~\cite{shi2023make}, DP-FedAvg~\cite{mcmahan2018learning}, and DP-BLUR-LUS~\cite{chen2022fedtune}. Since all the comparison methods are based on differential privacy, we set $\epsilon=1$ for all of them. Setting $\epsilon=1$ guarantees a utility-privacy trade-off. For a fair comparison, we considered $\delta=0.001$ for all comparison methods. Also, since the performance of DP-BLUR-LUS~\cite{chen2022fedtune} depends on another factor which is called sparsification ratio ($c$), we considered the default value of their implementation, which is $c=1e-5$. We apply gradient clipping to ensure that each client's gradient norm is bounded by a fixed threshold \(B\). 
If $\|\mathbf{g}_i\|_2 > B$, the gradient \( \mathbf{g}_i \) is scaled to $ \frac{B}{\|\mathbf{g}_i\|_2} \mathbf{g}_i$. The clipping in all DP mechanisms in this paper is calculated by Equation~(\ref{eq:clipping}) with  $B=0.5$. This step prevents any single client’s update from disproportionately influencing the global model leading to gradient explosions. Clipping bound is essential for all DP mechanisms. 

\begin{equation}
\label{eq:clipping}
    \mathbf{g}_i^{\text{clipped}} = \mathbf{g}_i \cdot \min\left(1, \frac{B}{\|\mathbf{g}_i\|_2}\right)
\end{equation}

Through extensive experiments, we aim to provide a comprehensive understanding of the Trustformer's capabilities and limitations, thereby offering insights that could guide future research and development in the intersection of federated learning, NLP, and secure computation. The subsequent sections detail the experimental setup, evaluation metrics, and a thorough analysis of the results obtained.

\subsection{Experimental Setup}
In this section, we describe the experimental setup used to evaluate the performance of the proposed Trustformer for the translation task from Russian to English. The setup encompasses the datasets, model architecture, federated learning configuration, evaluation metrics, and considerations for reproducibility.

\subsubsection{Computation Environment}
All experiments were conducted using the high-performance computing resources provided by the Digital Research Alliance of Canada~\cite{digitalresearchalliance}

 \subsubsection{Dataset}
We utilize the WMT19~\cite{barrault2019findings} dataset for our translation experiments, focusing on the Russian-to-English translation task. The WMT19 dataset is a widely recognized benchmark in machine translation, providing high-quality parallel corpora for various language pairs. The dataset is preprocessed and tokenized for our experiments using the Cross-lingual Language Model (XLM) tokenizer~\cite{lample2019cross}, proposed by Facebook AI, which is well-suited for handling multilingual text. As mentioned in Table~\ref{tab:experimental_settings}, we used 20,000 distinct records of this dataset to train each client. Since we considered three clients in our work, in total, we have 60,000 records in all clients. Also, we keep a held-out test set of 1,000 records for testing the approximated global model. 

\subsubsection{Model Architecture}
The core of our experimental setup is a Transformer model adapted for federated learning and Intel SGX. The Transformer architecture is known for its effectiveness in sequence-to-sequence tasks such as translation, owing to its self-attention mechanisms and parallelizable structure. The model parameters are initialized using standard practices, and the architecture is designed to balance performance and computational efficiency, making it suitable for deployment in a federated learning environment. Table~\ref{tab:experimental_settings} lists the hyper-parameters used in our experiments. In the following, we briefly explain important key parameters of the Trustformer architecture. 

\begin{table}[h]
\centering
\caption{Experimental setup}
\label{tab:experimental_settings}
  \scalebox{0.8}{
\begin{tabular}{ll}
 \toprule
\textbf{Parameter} & \textbf{Value} \\
\midrule
Source Vocabulary Size & 250,000 \\
Target Vocabulary Size & 250,000 \\
Model Dimension & 256 \\
Number of Attention Heads & 8 \\
Number of Layers & 6 \\
Feed-Forward Dimension & 512 \\
Maximum Sequence Length & 50 \\
Dropout Rate & 0.1 \\
Optimizer & Adam\\
Loss Function & Cross Entropy \\
Dataset size in each client & 20,000 \\
Test Dataset size & 1,000 \\
Batch Size & 20 \\
Number of Epochs & 10 \\
GPU Device & Tesla P100-PCIE-12GB \\
Number of Clients & 3 \\
Learning Rate & 0.001 \\
Weight Decay & 0.01 \\
Aggregation Frequency & 1 \\
\bottomrule
\end{tabular}}
\end{table}

\noindent\textbf{Source and Target Vocabulary Size}: They Specify the number of unique tokens in the source and target vocabularies, respectively. They set to 250,000. The Cross-lingual Language Model (XLM) tokenizer is used to preprocess text by converting it into a sequence of tokens the model can understand. It is designed to handle multiple languages, making it suitable for translation or cross-lingual understanding tasks. The tokenizer efficiently manages large vocabularies and applies Byte Pair Encoding (BPE) to split words into subword units, which helps capture the nuances of different languages and improves the model's ability to handle rare or unseen words.
    
    \noindent\textbf{Model Dimension}: The dimension of the model's layers is set to 256. It refers to the size of the hidden layers and the embedding space in the Transformer model. It determines the depth of the representations learned by the model and directly influences its capacity and performance.
    
    \noindent\textbf{Number of Attention Heads}: The number of parallel attention heads in the multi-head attention mechanism is set to 8. In the context of Transformer models, the number of attention heads refers to the number of parallel attention mechanisms, or "heads," used in the multi-head attention mechanism. Multi-head attention is a core component of Transformer architectures, such as those used in models like BERT, GPT, and the original Transformer model.
    
   \noindent\textbf{Number of Layers}: The number of layers in the Transformer model is set to 6. The number of layers in a Transformer model refers to the depth of the network, indicating how many layers of Transformers (consisting of self-attention and feed-forward sub-layers) the model has. This parameter is crucial for determining the model's capacity to learn and represent complex patterns in the data. 
    
    \noindent\textbf{Feed-Forward Dimension}: The dimension of the feed-forward layer is set to 512. The feed-forward dimension refers to the size of the intermediate layer in the position-wise feed-forward networks within each Transformer layer. This parameter is critical for determining the model's capacity to learn and process complex features of the input data.

    \noindent\textbf{Number of Clients}: The number of clients participating in training the Transformer is set to 3.
    
   \noindent\textbf{Aggregation Frequency}: The frequency of aggregating the model parameters is set to 1 epoch, which means after each round of local training, we do aggregation of the model using our proposed aggregation method. Doing aggregation after each iteration is not really necessary due to recent research showing that client's models in FL can have only one aggregation~\cite{guha2019one}. However, we considered the worst case of aggregation in our experiments.

\subsubsection{Additional Security Measures}
Our federated learning framework involves multiple decentralized devices (or clients) equipped with an Intel SGX enclave to ensure secure computation. Each client possesses a subset of the dataset and performs local training within the secure enclave provided by Intel SGX. The enclaves protect the data and model parameters during computation. As illustrated in  Figure~\ref{fig:architecture}, after $k$-means clustering, clients attest the server, then seal their local centroids and send them to the server. The server unseals the local centroids, aggregates all local centroids, and averages them to reach the global centroids. Afterward, the server seals the global centroids and sends them back to all clients for the next round of training. We used the Trustformer aggregation mechanism as detailed in Figure~\ref{fig:aggregation} to aggregate updates from all clients and converge to the global model. In this paper,  we simulated Intel SGX using a fixed key and AES cryptography technique to avoid the complexities of Intel SGX Software Development Kit (SDK) and ensure reproducibility for readers. This setup simplifies the implementation while maintaining the security principles of Intel SGX.

\subsection{Evaluation Metrics}
To assess the performance of the proposed Trustformer framework, we use a comprehensive set of evaluation metrics focusing on translation quality and model efficiency. In this section, we briefly explain these metrics. Since there is no fixed translation for a sentence, common evaluation metrics like accuracy are not useful in translation task evaluation. Therefore, we need other evaluation metrics that are not very common and can assess how similar the meaning of the translated output sequence is to the ground-truth sequence. These metrics are listed below, and all of them range from 0 to 1, where 1 indicates the high quality of translation, and 0 indicates the low quality of translation. 

\noindent\textbf{Bilingual Evaluation Understudy (BLEU)}:
Measures the precision of n-grams in the translated text compared to a reference translation. Higher BLEU scores~\cite{post2018call} indicate better translation quality.
% \subsubsection{Recall-Oriented Understudy for Gisting Evaluation (ROUGE):} Evaluates the overlap of n-grams between the translated text and the reference. We report ROUGE-1, ROUGE-2, and ROUGE-L scores to capture the precision and recall of unigrams, bigrams, and trigrams, respectively.

\noindent\textbf{Metric for Evaluation of Translation with Explicit ORdering (METEOR)}: Considers precision, recall, and synonymy, providing a more holistic measure of translation quality~\cite{denkowski2014meteor}.

\noindent\textbf{Bidirectional Encoder Representations from Transformers F1 (BERT F1)}: BERT is a pre-trained Transformer model designed to understand the context of words in a sentence by looking at both directions (left-to-right and right-to-left) simultaneously. The F1 score is a harmonic mean of precision and recall, providing a single metric that balances false positives and false negatives. It is particularly useful in scenarios where there is an imbalance in classes. The BERT F1~\cite{zhang2019bertscore} score explicitly evaluates how well the Transformer model is performing on the translation task. It is an essential metric for determining the model's practical usability and effectiveness in real-world applications, providing a balanced view of its precision and recall capabilities using the BERT pre-trained model. 

\noindent\textbf{Training Time:} This metric measures the time taken to train the global model across different configurations, providing insights into the computational efficiency of the federated learning setup.

\noindent\textbf{Model Size: } This metric measures the size of the model that should be transmitted between one client and the server. This metric is also known as network overhead in our work.

The above metrics provide a comprehensive evaluation of the Trustformer's capabilities in a translation task, allowing us to analyze its effectiveness and identify areas for improvement.

\subsection{Evaluation Results}

\begin{figure}[htbp]
  \centering
  \includegraphics[width=.5\textwidth]{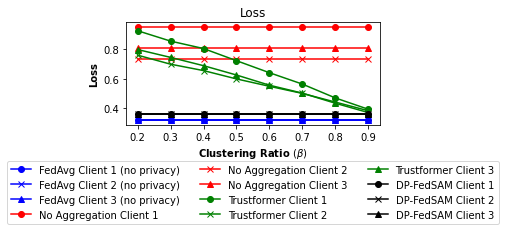} % Replace 'example-image' with the filename of your image
  \caption{Loss value in three clients with the varying number of clusters. FedAvg is the baseline method which provides no privacy. However, DP-FedSAM and Trustformer are privacy-preserving methods. }
  \label{fig:loss1}
\end{figure}

\begin{figure}[htbp]
  \centering
  \includegraphics[width=.5\textwidth]{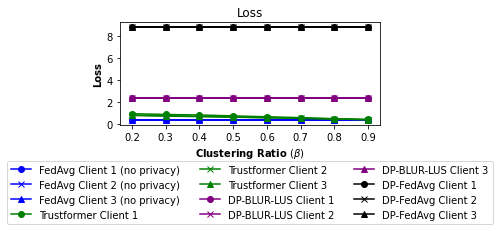} % Replace 'example-image' with the filename of your image
  \caption{Loss value in three clients with the varying number of clusters. FedAvg is the baseline method which provides no privacy. However, DP-FedAvg, DP-BLUR-LUS, and Trustformer are privacy-preserving methods. }
  \label{fig:loss2}
\end{figure}

\subsubsection{Loss against baseline}
\label{subsec:loss}
We conducted multiple experiments with various clustering ratios ($\beta$) to prove that our model works as expected according to Theorem~\ref{theo:centroids}. Figure~\ref{fig:loss1} illustrates how our proposed model converges to the global model. The overlap of the loss plots of clients shows that the clients are converging to the same model, resulting in the same loss when each client model is tested on the same dataset. 
%without revealing the global model to the server since the server has only the centroids of client updates. 
The $X$ axis shows the value of $\beta$, which can change the number of clusters as per Equation (\ref{eq:no-c}). We compared our proposed model against the baseline FedAvg~\cite{mcmahan2017communication} algorithm with no privacy.  In addition to the baseline, we compared our model to a scenario without client aggregation. In this case, each client trains its local model independently, with no aggregation between them, but they all use the same validation data for evaluation. 

Looking at Figure~\ref{fig:loss1}, we observe that using the same validation dataset on three clients, with no aggregation, there are many different loss values between the clients, and they are not converging to a single value. DP-FedSAM~\cite{shi2023make} shows a loss value worse than the baseline, but the convergence to the global model is guaranteed since the plots of all clients are overlapped. The loss value of the DP-FedSAM is close to that of the Trustformer when $\beta=0.9$. However, DP-FedSAM has higher communication overhead, as shown in subsection~\ref{sec:modelsize}.  In Trustformer, as shown in Figure~\ref{fig:loss1}, increasing the value of $\beta$ brings us closer to the baseline. However, this can also increase the model's size as per subsection~\ref{sec:modelsize}. Figure~\ref{fig:loss1} reveals that when the number of clusters is high enough (e.g. $\beta=0.9$), all clients will have the same loss value, meaning they obtain the global model without revealing the total model weight. This figure also experimentally underscores the correctness of Theorem 1, in which we proved that when we increase the number of clusters to the number of records, we will have the convergence equal to FedAvg.

Figure~\ref{fig:loss2} also illustrates the loss values obtained by Trustformer and other comparison methods. In Figure~\ref{fig:loss2}, DP-FedAvg~\cite{mcmahan2018learning} depicts the convergence to the global model. In contrast, higher loss leads to low translation quality, as shown in subsection~\ref{sec:quality}. DP-BLUR-LUS~\cite{chen2022fedtune} provides lower loss than DP-FedAvg during training and complete convergence between clients as expected. However, DP-BLUR-LUS demonstrates lower quality compared to Trustformer, as discussed in subsection subsection~\ref{sec:quality}. The reported loss values in Figure~\ref{fig:loss1} and Figure~\ref{fig:loss2} are the loss values of 1000 held-out test data during the training of clients. This means that the test data on all clients are the same. We chose the same test data to see how each client's model behaves to measure if any convergence happens among the local models.

\subsubsection{Model size}
\label{sec:modelsize}

\begin{table}[h]
\centering
\caption{The volume of data in transit between each client and the server per one epoch with various numbers of clusters.}
\label{tab:data_size}
  \scalebox{0.8}{
\begin{tabular}{l c c}
\toprule
\textbf{Method} & \textbf{Data size} & \textbf{Private?} \\
\midrule
\textbf{Baseline (FedAvg)~\cite{mcmahan2017communication}}  & 782,003 KB & No \\
\textbf{DP-FedAvg~\cite{mcmahan2018learning}} & 782,003 KB & Yes \\
\textbf{DP-FedSAM~\cite{shi2023make}} & 782,003 KB & Yes \\
\textbf{DP-BLUR-LUS (with $c=1e-5$)~\cite{cheng2022differentially}} & 515,952 KB & Yes \\
\midrule
\textbf{Trustformer ($\beta=0.1$)} & 78,019 KB & Yes \\
\textbf{Trustformer ($\beta=0.2$)} & 156,127 KB & Yes \\
\textbf{Trustformer ($\beta=0.3$)} & 234,140 KB & Yes \\
\textbf{Trustformer ($\beta=0.4$)} & 312,248 KB & Yes \\
\textbf{Trustformer ($\beta=0.5$)} & 390,368 KB & Yes \\
\textbf{Trustformer ($\beta=0.6$)} & 468,380 KB & Yes \\
\textbf{Trustformer ($\beta=0.7$)} & 546,488 KB & Yes \\
\textbf{Trustformer ($\beta=0.8$)} & 624,500 KB & Yes \\
\textbf{Trustformer ($\beta=0.9$)} & 702,608 KB & Yes \\
\bottomrule
\end{tabular}}
\end{table}

\begin{table*}[h]
\centering
\caption{Time in seconds used for training all clients to converge to the global model for one epoch.}
\label{tab:training_time}
  \scalebox{0.8}{
\begin{tabular}{l c c c c c}
\toprule
\textbf{Method} & \textbf{Training time (sec)} & \textbf{clustering time (sec)} & \textbf{Aggregation time (sec)} & \textbf{total (sec)} & privacy?\\
\midrule
\textbf{Baseline (FedAvg)~\cite{mcmahan2017communication}} & 479 & 0 & 17 & 496 & No \\
\textbf{DP-FedAvg($\epsilon=1$)~\cite{mcmahan2018learning}} & 563 & 0 & 11 & 664 & Yes \\
\textbf{DP-FedSAM($\epsilon=1$)~\cite{shi2023make}} & 442 & 0 & 10 & 452 & Yes\\
\textbf{DP-BLUR-LUS($\epsilon=1, c=1e-5$)~\cite{cheng2022differentially}} & 417 & 0 & 4 & 421 & Yes\\
\midrule
\textbf{Trustformer ($\beta=0.1$)}  & 491 & 564 & 1 & 1,056 & Yes\\
\textbf{Trustformer ($\beta=0.2$)}  & 484 & 1,287 & 4 & 1,775 & Yes\\
\textbf{Trustformer ($\beta=0.3$)}  & 478 & 1,090 & 5 & 1,573 & Yes\\
\textbf{Trustformer ($\beta=0.4$)}  & 473 & 2,310 & 7 & 2,790 & Yes\\
\textbf{Trustformer ($\beta=0.5$)}  & 479 & 3,291 & 10 & 3,780 & Yes\\
\textbf{Trustformer ($\beta=0.6$)}  & 475 & 3,438 & 11 &  3,924 & Yes \\
\textbf{Trustformer ($\beta=0.7$)}  & 476 & 3,296 & 15 & 3,787 & Yes \\
\textbf{Trustformer ($\beta=0.8$)}  & 472 & 2,502 & 17 & 2,991 & Yes\\
\textbf{Trustformer ($\beta=0.9$)}  & 477 & 1,156 & 19 & 1,652 & Yes \\

\bottomrule
\end{tabular}}
\end{table*}

\begin{table}[h]
\centering
\caption{Total time in hours to converge to the global model based on various aggregation frequencies.}
\label{tab:total_time}
  \scalebox{0.8}{
\begin{tabular}{p{3cm} c c c c c}
\toprule
\textbf{Method} &\textbf{$Agg_{fr}=1$}&\textbf{$Agg_{fr}=2$} &\textbf{$Agg_{fr}=5$} &\textbf{$Agg_{fr}=10$}  \\
\midrule
\textbf{Baseline (FedAvg)~\cite{mcmahan2017communication}} & 2.08 & 1.71 &1.48  & 1.40 \\
\textbf{DP-FedAvg($\epsilon=1$)~\cite{mcmahan2018learning}} & 1.87 & 1.72 & 1.62 & 1.59 \\
\textbf{DP-FedSAM($\epsilon=1$)~\cite{shi2023make}} & 1.50 & 1.36 &  1.28 & 1.25 \\
\textbf{DP-BLUR-LUS($\epsilon=1, c=1e-5$)~\cite{cheng2022differentially}} & 1.27 & 1.21 & 1.18 & 1.17 \\
\midrule
\textbf{Trustformer ($\beta=0.1$)}  & 2.90 & 2.16 & 1.68 & 1.52 \\
\textbf{Trustformer ($\beta=0.2$)}  & 5.03 & 3.19 & 2.08 & 1.71 \\
\textbf{Trustformer ($\beta=0.3$)}  & 4.49 & 2.90 & 1.96 & 1.64 \\
\textbf{Trustformer ($\beta=0.4$)}  & 7.93& 4.62 & 2.63 & 1.97 \\
\textbf{Trustformer ($\beta=0.5$)}  & 10.78 & 4.85 & 3.22 & 2.27 \\
\textbf{Trustformer ($\beta=0.6$)}  & 11.17 & 6.24 & 3.29& 2.30 \\
\textbf{Trustformer ($\beta=0.7$)}  & 10.90 & 6.11 & 3.23& 2.28 \\
\textbf{Trustformer ($\beta=0.8$)}  & 8.64 & 4.99 & 2.77 & 2.04 \\
\textbf{Trustformer ($\beta=0.9$)}  & 5.05 & 3.19 & 2.07 & 1.69 \\

\bottomrule
\end{tabular}}
\end{table}

Table \ref{tab:data_size} underscores our approach's efficacy in data transmission between clients and the server. This significant difference is because we send each trainable layer's centroids of clustering results. This idea significantly reduces the amount of data that has to be passed between entities. Considering Figure~\ref{fig:loss1}, Figure~\ref{fig:loss2}, and Table~\ref{tab:data_size}, we can conclude that even if we consider $\beta=0.9$, we still have about 80MB (10\%), less transmission overhead than the FedAvg, DP-FedAvg, and DP-FedSAM. Although Trustformer shows a larger data size with $\beta=0.9$ compared to DP-BLUR-LUS, the quality of the results of DP-BLUR-LUS as shown in subsection \ref{sec:quality} is significantly lower than that of Trustformer.

\subsubsection{Training Time}

Table~\ref{tab:training_time}, and Table~\ref{tab:total_time}, compare our approach with its counterparts in terms of training time. Table~\ref{tab:training_time} shows each component's training time per epoch, and Table~\ref{tab:total_time} shows the total training time ($T_{total}$) when we change the aggregation frequency. In these two tables, We increase the number of clusters to observe the differences between Trustformer and the other methods. The total reported time in Table~\ref{tab:total_time} is calculated using Equation~(\ref{eq:training_time}). 
\begin{equation}
\begin{split}
T_{total} = & \left(e.\max_{\forall i \in \{1, 2, \dots, n\}} (T_{l}^{(i)})\right) \\
& + \left(\frac{e}{Agg_{fr}}.\max_{\forall i \in \{1, 2, \dots, n\}}(T_{c}^{(i)} + T_{agg})\right)
\end{split}
\label{eq:training_time}
\end{equation}
\noindent where $n$ is the number of clients, $e$ is the number of epochs, which in our experiments equals 10 (i.e. $e=10$). $T_l^{(i)}$ is the training time of the local model in $i$-th client, $T_c^{(i)}$ is the clustering time after each round of training in client $i$, and $T_{agg}$ is the time used for aggregation of the local centroids by the server for each epoch. $Agg_{fr}$ is the aggregation frequency used by our work, which shows how often the aggregation should happen.  
We did not consider operating system limitations and workload that may affect the reported time. In Equation~\ref{eq:training_time}, we observe that to have a trusted aggregation between clients and the server without leaking information, we need to do more calculations on the client's side, which is time-consuming based on two variables: I) the number of clusters being sent between the clients and the server, and II) the frequency of aggregations between clients and the server. The clustering time in our approach varies significantly with changes in the clustering ratio. For instance, when $\beta=0.6$, we have the highest clustering time equal to 3,438 seconds, but when we increase the number of clusters to $\beta=0.9$, the clustering time reduces to 1,156 seconds. This is justified by the fact that since we have more centroids in $\beta=0.9$, the locations of the centroids are not changed after some certain rounds of iterations in $k$-means clustering, and it converges to a result faster than when we have less number of clusters in $\beta=0.6$. %This shows the worst case of Trustformer by having the aggregation frequency $Agg_{fr}=1$, which means for every epoch, the aggregation happens. Because we have 10 epochs, then 10 aggregations happen. If we do aggregation one or two times during the training process, the total amount of training time will be reduced. 

Table~\ref{tab:total_time} tabulates the total time of training the whole model. When we have $Agg_{fr}=1$, since there are 10 epochs, we do aggregation 10 times. Therefore, the Trustformer training time is much higher than the state-of-the-art methods. However, if we consider one-shot learning, which means doing aggregation one time in training ($Agg_{fr}=10$), we see much less difference between the Trustformer and the state-of-the-art methods in terms of training time. Whenever $Agg_{fr}=e$, it means every 10 epochs, we have one aggregation, which equals applying one-shot learning~\cite{guha2019one}.

\subsubsection{Federated Training Quality}
\label{sec:quality}
In translation tasks, using accuracy as a metric is insufficient due to the nuanced nature of languages~\cite{chatzikoumi2020evaluate}. Accuracy typically measures the exact match between predicted and reference outputs, which is too rigid for evaluating translations where multiple valid translations can exist for a single source sentence. Instead, alternative metrics like BLEU score~\cite{post2018call}, METEOR~\cite{denkowski2014meteor}, or BERT F1~\cite{zhang2019bertscore} are more appropriate. BLEU score assesses the quality of the translation by comparing n-grams of the candidate translation to n-grams of the reference translations, capturing partial matches and fluency. METEOR goes further by aligning the words in the candidate and reference translations based on synonyms, stemming, and paraphrasing, offering a more flexible evaluation of translation quality. BERT F1 leverages semantic similarity by using a pre-trained language model to compare the embeddings of the candidate and reference translations, ensuring that the meaning is preserved even if the exact wording differs. These metrics provide a more comprehensive evaluation of translation quality, reflecting the variability and richness of human language. Figure~\ref{fig:performance} highlights the performance of our approach compared to state-of-the-art methods. As depicted, Trustformer yields better results in terms of BERT F1 compared to the other methods and produces results comparable to  DP-FedSAM when evaluated using BLEU  and METEOR scores. We can conclude that although the results are similar to those of DP-FedSAM, Trustformer has significantly lower communication overhead, as shown in Table~\ref{tab:data_size}.

\section{Security Discussion}
\label{sec:discussion}
All exchanged data between the clients' enclaves and the server's enclave is encrypted through the Trusted Execution Environments (TEE) using SGX. Only clients can access their data. We acknowledge that certain security attacks, such as page fault attacks, target TEE~\cite{shinde2016preventing}. However, additional security measures can mitigate these types of attacks~\cite{lang2022mole}, though they are beyond the scope of this paper. To tackle the first threat model discussed in subsection~\ref{sec:threat}, TEE ensures that centroids being transmitted are not accessible by the central server and external attackers. In addition,  TEE ensures that all clients are legitimate and that their centroid updates have not been tampered with.  For the second threat model discussed in~\ref{sec:threat}, where the server is compromised, the data transferred between clients and the server consists of the centroids of the clusters, which is a generalized version of individual model weights and not exactly the model weights. In addition, TEE protects the exchanged centroids together with the required calculations computed by the central server. Therefore, the untrusted server cannot infer any useful information. 

\begin{figure}[htbp]
  \centering
  \includegraphics[width=.46\textwidth]{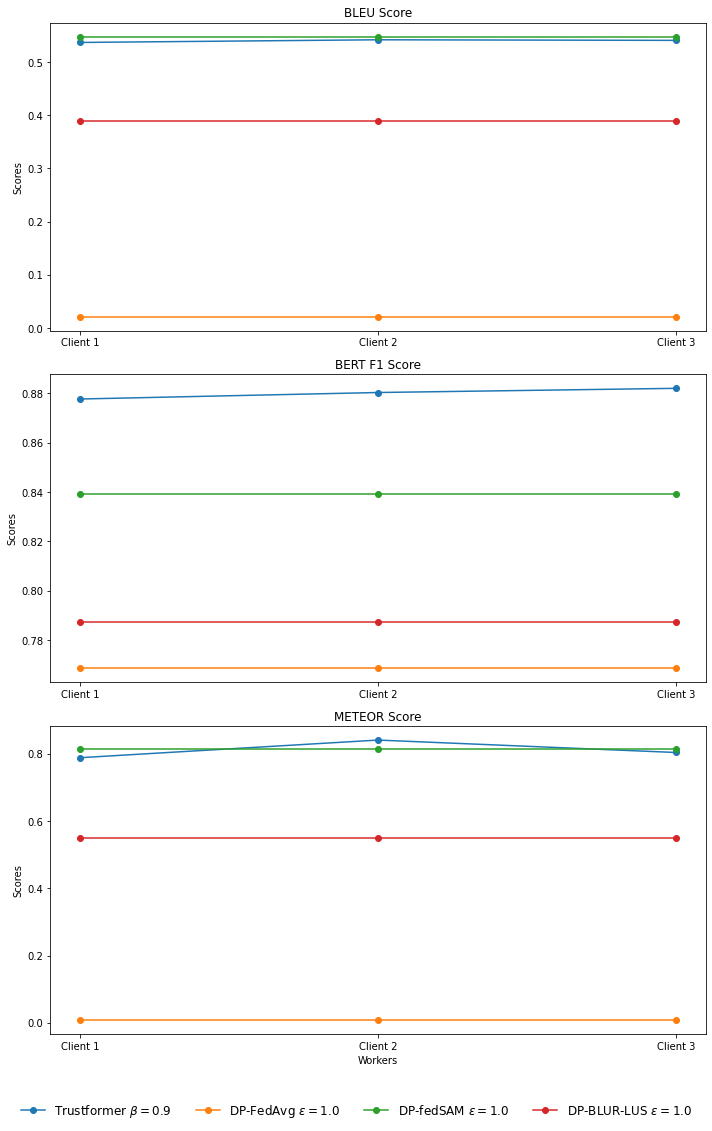} % Replace 'example-image' with the filename of your image
  \caption{Global model quality on each client on 1000 held-out test data.}
  \label{fig:performance}
\end{figure}

%\section{Theories}
%\label{sec:theory}
%\textbf{Theorem 1}
%Given $l$ matrices $M_1, M_2, \ldots, M_l$  each of size  $NO_c \times f$, if we cluster these matrices separately and consider the centroids, then average the centroids and move cluster points according to the average, as the number of clusters $k$ increases, this process is equivalent to averaging the matrices element-wise (FedAvg).

\section{Conclusion}
\label{sec:conclusion}
In this paper, we introduced Trustformer, a trusted federated Transformer. Trustformer incorporates federated learning into Transformers, leading to a secure environment for training a Transformer from scratch. Trustformer applies $k$-means on every layer of Transformers inside each client and averages the centroids of each layer in the central server.  This approach allows Trustformer to not only prevent sharing the model weights,  but also to prevent communication overhead by using centroids of the models instead of the model weights. To further enhance security, Trustformer utilizes Intel SGX for remote attestation and secure transfer of centroids between clients and the central server. Trustformer opens a new research direction for the secure aggregation of federated learning for large models. Future research direction can lie in investigating how well Trustformer can perform on personalized federated learning models~\cite{chen2022self}, as well as finding new and less complex techniques to reduce the size of centroids being aggregated with dimensionality reduction techniques.

\bibliographystyle{acm}

\begin{thebibliography}{10}

\bibitem{abbasi2023comparative}
{\sc Abbasi~Tadi, A., Dayal, S., Alhadidi, D., and Mohammed, N.}
\newblock Comparative analysis of membership inference attacks in federated and centralized learning.
\newblock {\em Information 14}, 11 (2023), 620.

\bibitem{acar2021federated}
{\sc Acar, D. A.~E., Zhao, Y., Navarro, R.~M., Mattina, M., Whatmough, P., and Saligrama, V.}
\newblock Federated learning based on dynamic regularization.
\newblock In {\em International Conference on Learning Representations\/} (2021).

\bibitem{agarwal2018cpsgd}
{\sc Agarwal, N., Suresh, A.~K., Yu, F.~X., Kumar, S., and McMahan, H.~B.}
\newblock cpsgd: Communication-efficient and differentially-private distributed sgd.
\newblock In {\em Advances in Neural Information Processing Systems\/} (2018), pp.~7564--7575.

\bibitem{ahuja2020federated}
{\sc Ahuja, K., Nguyen, M.~V., Islam, M. A.~A., Badie-Modiri, A., and Rabinovich, M.}
\newblock Federated learning with adaptive communication compression.
\newblock In {\em Advances in Neural Information Processing Systems\/} (2020), vol.~33, pp.~464--476.

\bibitem{trustformergithub}
{\sc anonymous}.
\newblock Trustformer, 2024.
\newblock https://github.com/anonymoustrustformer/trustformer.

\bibitem{barrault2019findings}
{\sc Barrault, L., Bojar, O., Costa-Jussa, M.~R., Federmann, C., Fishel, M., Graham, Y., Haddow, B., Huck, M., Koehn, P., Malmasi, S., et~al.}
\newblock Findings of the 2019 conference on machine translation (wmt19).
\newblock ACL.

\bibitem{bu2022scalable}
{\sc Bu, Z., Mao, J., and Xu, S.}
\newblock Scalable and efficient training of large convolutional neural networks with differential privacy.
\newblock {\em Advances in Neural Information Processing Systems 35\/} (2022), 38305--38318.

\bibitem{chatzikoumi2020evaluate}
{\sc Chatzikoumi, E.}
\newblock How to evaluate machine translation: A review of automated and human metrics.
\newblock {\em Natural Language Engineering 26}, 2 (2020), 137--161.

\bibitem{chen2022self}
{\sc Chen, H., Ding, J., Tramel, E.~W., Wu, S., Sahu, A.~K., Avestimehr, S., and Zhang, T.}
\newblock Self-aware personalized federated learning.
\newblock {\em Advances in Neural Information Processing Systems 35\/} (2022), 20675--20688.

\bibitem{chen2022fedtune}
{\sc Chen, J., Xu, W., Guo, S., Wang, J., Zhang, J., and Wang, H.}
\newblock Fedtune: A deep dive into efficient federated fine-tuning with pre-trained transformers.
\newblock {\em arXiv preprint arXiv:2211.08025\/} (2022).

\bibitem{cheng2022differentially}
{\sc Cheng, A., Wang, P., Zhang, X.~S., and Cheng, J.}
\newblock Differentially private federated learning with local regularization and sparsification.
\newblock In {\em Proceedings of the IEEE/CVF conference on computer vision and pattern recognition\/} (2022), pp.~10122--10131.

\bibitem{dayal2023comparative}
{\sc Dayal, S., Alhadidi, D., Abbasi~Tadi, A., and Mohammed, N.}
\newblock Comparative analysis of membership inference attacks in federated learning.
\newblock In {\em Proceedings of the 27th International Database Engineered Applications Symposium\/} (2023), pp.~185--192.

\bibitem{denkowski2014meteor}
{\sc Denkowski, M., and Lavie, A.}
\newblock Meteor universal: Language specific translation evaluation for any target language.
\newblock In {\em Proceedings of the ninth workshop on statistical machine translation\/} (2014), pp.~376--380.

\bibitem{digitalresearchalliance}
{\sc {Digital Research Alliance of Canada}}.
\newblock Digital research alliance of canada, 2024.
\newblock Accessed: 2024-06-10.

\bibitem{dosovitskiy2021image}
{\sc Dosovitskiy, A., Beyer, L., Kolesnikov, A., Weissenborn, D., Zhai, X., Unterthiner, T., Dehghani, M., Minderer, M., Heigold, G., Gelly, S., et~al.}
\newblock An image is worth 16x16 words: Transformers for image recognition at scale.
\newblock In {\em International Conference on Learning Representations\/} (2021).

\bibitem{esser2021taming}
{\sc Esser, P., Rombach, R., and Ommer, B.}
\newblock Taming transformers for high-resolution image synthesis.
\newblock In {\em Proceedings of the IEEE/CVF Conference on Computer Vision and Pattern Recognition\/} (2021), pp.~12873--12883.

\bibitem{fowl2022decepticons}
{\sc Fowl, L., Geiping, J., Reich, S., Wen, Y., Czaja, W., Goldblum, M., and Goldstein, T.}
\newblock Decepticons: Corrupted transformers breach privacy in federated learning for language models.
\newblock {\em arXiv preprint arXiv:2201.12675\/} (2022).

\bibitem{gong2023gradient}
{\sc Gong, H., Jiang, L., Liu, X., Wang, Y., Gastro, O., Wang, L., Zhang, K., and Guo, Z.}
\newblock Gradient leakage attacks in federated learning.
\newblock {\em Artificial Intelligence Review 56}, Suppl 1 (2023), 1337--1374.

\bibitem{guha2019one}
{\sc Guha, N., Talwalkar, A., and Smith, V.}
\newblock One-shot federated learning.
\newblock In {\em Proceedings of the 36th International Conference on Machine Learning\/} (2019).

\bibitem{ji2021learning}
{\sc Ji, S., Luo, T., Chen, H., Li, J., and Jordan, M.~I.}
\newblock Learning private neural language modeling with attentive aggregation.
\newblock {\em arXiv preprint arXiv:2106.07821\/} (2021).

\bibitem{jumper2021highly}
{\sc Jumper, J., Evans, R., Pritzel, A., Green, T., Figurnov, M., Ronneberger, O., Tunyasuvunakool, K., Bates, R., {\v{Z}}{\'\i}dek, A., Potapenko, A., et~al.}
\newblock Highly accurate protein structure prediction with alphafold.
\newblock {\em Nature 596}, 7873 (2021), 583--589.

\bibitem{kang2020incentive}
{\sc Kang, J., Xiong, Z., Niyato, D., Yu, P., and Liang, Y.-C.}
\newblock Incentive design for efficient federated learning in mobile networks: A contract theory approach.
\newblock {\em IEEE Journal on Selected Areas in Communications 39}, 1 (2020), 152--165.

\bibitem{khan2022transformers}
{\sc Khan, S., Naseer, M., Hayat, M., Zamir, S.~W., Khan, F.~S., and Shah, M.}
\newblock Transformers in vision: A survey.
\newblock {\em ACM Computing Surveys (CSUR) 54}, 10s (2022), 1--41.

\bibitem{khashabi2020unifiedqa}
{\sc Khashabi, D., Min, S., Khot, T., Sabharwal, A., Tafjord, O., Clark, P., and Hajishirzi, H.}
\newblock Unifiedqa: Crossing format boundaries with a single qa system.
\newblock In {\em Findings of the Association for Computational Linguistics: EMNLP 2020\/} (2020), pp.~1896--1907.

\bibitem{kingma2014adam}
{\sc Kingma, D.~P., and Ba, J.}
\newblock Adam: A method for stochastic optimization.
\newblock {\em arXiv preprint arXiv:1412.6980\/} (2014).

\bibitem{knott2021crypten}
{\sc Knott, B., Venkataraman, S., Hannun, A., Sengupta, S., Ibrahim, M., and van~der Maaten, L.}
\newblock Crypten: Secure multi-party computation meets machine learning.
\newblock {\em Advances in Neural Information Processing Systems 34\/} (2021), 4961--4973.

\bibitem{lalitha2019fully}
{\sc Lalitha, A.~K., Ghosh, S., Basu, S., Suresh, A.~K., and Kannan, R.}
\newblock Fully decentralized federated learning.
\newblock In {\em Third Workshop on Bayesian Deep Learning (NeurIPS)\/} (2019).

\bibitem{lample2019cross}
{\sc Lample, G., and Conneau, A.}
\newblock Cross-lingual language model pretraining.
\newblock In {\em Advances in Neural Information Processing Systems\/} (2019), pp.~7059--7069.

\bibitem{lang2022mole}
{\sc Lang, F., Wang, W., Meng, L., Lin, J., Wang, Q., and Lu, L.}
\newblock Mole: Mitigation of side-channel attacks against sgx via dynamic data location escape.
\newblock In {\em Proceedings of the 38th Annual Computer Security Applications Conference\/} (2022), pp.~978--988.

\bibitem{lee2022privacy}
{\sc Lee, J.-W., Kang, H., Lee, Y., Choi, W., Eom, J., Deryabin, M., Lee, E., Lee, J., Yoo, D., Kim, Y.-S., et~al.}
\newblock Privacy-preserving machine learning with fully homomorphic encryption for deep neural network.
\newblock {\em iEEE Access 10\/} (2022), 30039--30054.

\bibitem{li2020convergence}
{\sc Li, X., Huang, K., Yang, W., Wang, S., and Zhang, Z.}
\newblock On the convergence of fedavg on non-iid data.
\newblock In {\em International Conference on Learning Representations\/} (2020).

\bibitem{lin2020deep}
{\sc Lin, Y., Han, S., Mao, H., Wang, Y., and Dally, W.~J.}
\newblock Deep gradient compression: Reducing the communication bandwidth for distributed training.
\newblock In {\em International Conference on Learning Representations\/} (2020).

\bibitem{mcmahan2017communication}
{\sc McMahan, H.~B., Moore, E., Ramage, D., Hampson, S., and y~Arcas, B.~A.}
\newblock Communication-efficient learning of deep networks from decentralized data.
\newblock In {\em Artificial Intelligence and Statistics\/} (2017), PMLR, pp.~1273--1282.

\bibitem{mcmahan2018learning}
{\sc McMahan, H.~B., Ramage, D., Talwar, K., and Zhang, L.}
\newblock Learning differentially private recurrent language models.
\newblock In {\em International Conference on Learning Representations (ICLR)\/} (2018).

\bibitem{miao2023k}
{\sc Miao, S., Zheng, L., Liu, J., and Jin, H.}
\newblock K-means clustering based feature consistency alignment for label-free model evaluation.
\newblock In {\em Proceedings of the IEEE/CVF Conference on Computer Vision and Pattern Recognition\/} (2023), pp.~3299--3307.

\bibitem{phong2018privacy}
{\sc Phong, L.~T., Aono, Y., Hayashi, T., Wang, L., and Moriai, S.}
\newblock Privacy-preserving deep learning via additively homomorphic encryption.
\newblock {\em IEEE Transactions on Information Forensics and Security 13}, 5 (2018), 1333--1345.

\bibitem{post2018call}
{\sc Post, M.}
\newblock A call for clarity in reporting bleu scores.
\newblock {\em arXiv preprint arXiv:1804.08771\/} (2018).

\bibitem{raffel2020exploring}
{\sc Raffel, C., Shazeer, N., Roberts, A., Lee, K., Narang, S., Matena, M., Zhou, Y., Li, W., and Liu, P.~J.}
\newblock Exploring the limits of transfer learning with a unified text-to-text transformer.
\newblock In {\em Journal of Machine Learning Research\/} (2020), p.~140.

\bibitem{reddi2021adaptive}
{\sc Reddi, S.~J., Charles, Z., Zaheer, M., Garrett, Z., Rush, K., Konecn{\'y}, J., Kumar, S., McMahan, B., and Hsieh, C.-J.}
\newblock Adaptive federated optimization.
\newblock In {\em International Conference on Learning Representations\/} (2021).

\bibitem{shi2023make}
{\sc Shi, Y., Liu, Y., Wei, K., Shen, L., Wang, X., and Tao, D.}
\newblock Make landscape flatter in differentially private federated learning.
\newblock In {\em Proceedings of the IEEE/CVF Conference on Computer Vision and Pattern Recognition\/} (2023), pp.~24552--24562.

\bibitem{shinde2016preventing}
{\sc Shinde, S., Chua, Z.~L., Narayanan, V., and Saxena, P.}
\newblock Preventing page faults from telling your secrets.
\newblock In {\em Proceedings of the 11th ACM on Asia Conference on Computer and Communications Security\/} (2016), pp.~317--328.

\bibitem{tadi2024pppct}
{\sc Tadi, A.~A., Alhadidi, D., and Rueda, L.}
\newblock Pppct: Privacy-preserving framework for parallel clustering transcriptomics data.
\newblock {\em Computers in Biology and Medicine 173\/} (2024), 108351.

\bibitem{tadi2022nicasn}
{\sc Tadi, A.~A., Rueda, L., and Alhadidi, D.}
\newblock Nicasn: Non-negative matrix factorization and independent component analysis for clustering social networks.
\newblock In {\em Canadian AI\/} (2022).

\bibitem{truex2019hybrid}
{\sc Truex, S., Baracaldo, N., Anwar, A., Steinke, T., Ludwig, H., Weber, B., and Zhang, R.}
\newblock A hybrid approach to privacy-preserving federated learning.
\newblock In {\em Proceedings of the 12th ACM Workshop on Artificial Intelligence and Security\/} (2019), ACM, pp.~1--11.

\bibitem{vaswani2017attention}
{\sc Vaswani, A., Shazeer, N., Parmar, N., Uszkoreit, J., Jones, L., Gomez, A.~N., Kaiser, {\L}., and Polosukhin, I.}
\newblock Attention is all you need.
\newblock In {\em Advances in neural information processing systems\/} (2017), pp.~5998--6008.

\bibitem{wang2020federated}
{\sc Wang, H., Yurochkin, M., Sun, Y., Papailiopoulos, D., and Khazaeni, Y.}
\newblock Federated learning with matched averaging.
\newblock In {\em International Conference on Learning Representations\/} (2020).

\bibitem{NEURIPS2023_e4724af0}
{\sc Yang, X., Huang, W., and Ye, M.}
\newblock Dynamic personalized federated learning with adaptive differential privacy.
\newblock In {\em Advances in Neural Information Processing Systems\/} (2023), A.~Oh, T.~Naumann, A.~Globerson, K.~Saenko, M.~Hardt, and S.~Levine, Eds., vol.~36, Curran Associates, Inc., pp.~72181--72192.

\bibitem{zhang2020pegasus}
{\sc Zhang, J., Zhao, Y., Saleh, M., and Liu, P.~J.}
\newblock Pegasus: Pre-training with extracted gap-sentences for abstractive summarization.
\newblock In {\em International Conference on Machine Learning\/} (2020), pp.~11328--11339.

\bibitem{zhang2021vinvl}
{\sc Zhang, P., Li, X., Hu, X., Yang, J., Zhang, L., Wang, Y., and Gao, J.}
\newblock Vinvl: Revisiting visual representations in vision-language models.
\newblock In {\em Proceedings of the IEEE/CVF Conference on Computer Vision and Pattern Recognition\/} (2021), pp.~5579--5588.

\bibitem{zhang2019bertscore}
{\sc Zhang, T., Kishore, V., Wu, F., Weinberger, K.~Q., and Artzi, Y.}
\newblock Bertscore: Evaluating text generation with bert.
\newblock {\em arXiv preprint arXiv:1904.09675\/} (2019).

\end{thebibliography}

\appendix

\section*{Appendix}
\subsection*{Definition and notations}
\begin{enumerate}
    \item Total number of clients: $n$
    \item Each client $i$ has a local model parameter vector $w_i\in \mathbb{R}^r$. For simplicity, we assume there is only 1 hidden layer, meaning that $W^i_{local}=\{w_1^i\}$ in the local models. We show $w_1^i$ as $w_i$. Therefore, $w_i = [w_{i,1}, w_{i,2}, ..., w_{i,r}]$.
    \item Each client clusters their parameter vector $w_i$ into $No_c$ clusters using $k$-means clustering.
    \item For client $i$, clusters are denoted by $S_{i,j}$ for $j=1,2, .., No_c$.
    \item Centroids for client $i$: $c_{i,j}=\frac{1}{|S_{i,j}|} \Sigma_{l\in S_{i,j}} w_{i,l}$.
    \item Client $i$ sends centroids $c_i=[\hat{c}_{i,1}, \hat{c}_{i,2},\dots, \hat{c}_{i,No_c}]$ to the server
    \item Since we have just one layer in the clients, the $C_{global}=\{c_{avg,1}\}$. We show $c_{avg,1}$ as $c_{avg}$ for simplicity. The server aggregates centroids: $c_{avg,j}=\frac{1}{n}\Sigma_{i=1}^{n} c_{i,j}$ for $j= 1, 2, \dots,No_c$ and sends it to clients.
    \item Clients approximate the global model and since we have only 1 layer, $W_{global}=\{w_1\}$. For simplicity, we call the global weights on client $i$ as $w_i^{global}$. Clients approximate $w_i^{global}$ using $c_{i,j}$, $w_i$, and $c_{avg,j}$.
    \end{enumerate}

\section*{Theorem and Proofs}

\noindent \textbf{Theorem 1.}
%\label{theo:centroids}
In a federated learning setting where each client clusters their local model parameters \( w_i \in \mathbb{R}^r \) into \( No_c \) clusters and updates their model by adjusting parameters based on the differences between global and local centroids, the clients' updated models \( w_i^{\text{global}} \) converge to the Federated Averaging (FedAvg) global model \( w_{\text{FedAvg}} = \frac{1}{n} \sum_{i=1}^{n} w_i \) as \( No_c \to r \).

\subsection*{Supporting Lemmas and Corollaries}

\begin{lemma}[Parameter Adjustment Formula]
For each client \( i \) and parameter \( l \in S_{i,j} \) belonging to cluster \( j \), the updated parameter after Algorithm~\ref{alg:estimate-gm} is:
\[
w_{i,l}^{\text{global}} = w_{i,l} + (c_{\text{avg},j} - c_{i,j}),
\]
where:
\begin{align*}
c_{i,j} &= \frac{1}{|S_{i,j}|} \sum_{l \in S_{i,j}} w_{i,l} \quad \text{(Local centroid)}, \\
c_{\text{avg},j} &= \frac{1}{n} \sum_{i=1}^{n} c_{i,j} \quad \text{(Global centroid)}.
\end{align*}
\end{lemma}

\begin{proof}
This follows directly from the adjustment rule specified in the Algorithm~\ref{alg:estimate-gm}.
\end{proof}

\begin{lemma}[Average of Updated Parameters Equals FedAvg]
The average of the updated parameters across all clients equals the FedAvg global model:
\[
w_{l}^{\text{avg,global}} = \frac{1}{n} \sum_{i=1}^{n} w_{i,l}^{\text{global}} = w_{\text{FedAvg},l}, \quad \forall l \in \{1, 2, \dots, r\}.
\]
\end{lemma}

\begin{proof}
For any parameter \( l \):
\begin{align*}
w_{l}^{\text{avg,global}} &= \frac{1}{n} \sum_{i=1}^{n} \left( w_{i,l} + (c_{\text{avg},j} - c_{i,j}) \right ) \\
&= \left( \frac{1}{n} \sum_{i=1}^{n} w_{i,l} \right ) + c_{\text{avg},j} - \left( \frac{1}{n} \sum_{i=1}^{n} c_{i,j} \right ) \\
&= \frac{1}{n} \sum_{i=1}^{n} w_{i,l} \quad \text{(since \( c_{\text{avg},j} = \frac{1}{n} \sum_{i=1}^{n} c_{i,j} \))} \\
&= w_{\text{FedAvg},l}.
\end{align*}
\end{proof}

\begin{lemma}[Bound on Parameter Difference \( \delta_{i,l} \)]
Define:
\[
\delta_{i,l} = w_{\text{FedAvg},l} - w_{i,l}.
\]
Assuming that the parameters across clients satisfy \( |w_{i',l} - w_{i,l}| \leq B \) for all \( i, i' \), then:
\[
|\delta_{i,l}| \leq B.
\]
\end{lemma}

\begin{proof}
\begin{align*}
\delta_{i,l} &= \frac{1}{n} \sum_{i'=1}^{n} w_{i',l} - w_{i,l} \\
&= \frac{1}{n} \left( \sum_{i'=1}^{n} w_{i',l} - n w_{i,l} \right ) \\
&= \frac{1}{n} \sum_{i' \ne i}^{n} (w_{i',l} - w_{i,l} ).
\end{align*}
Taking absolute values and using triangle inequality:
\[
|w_{i,l}| \leq \frac{1}{n} \sum_{i' \ne i}^{n} |w_{i',l} - w_{i,l}| \leq \frac{n - 1}{n} B \leq B.
\]
\end{proof}

\begin{lemma}[Bound on Centroid Difference \( \eta_{i,j} \)]
Define:
\[
\eta_{i,j} = c_{i,j} - c_{\text{avg},j}.
\]
Assuming that the centroids satisfy \( |\mu_{i,j}| \leq C \) for all \( i, j \), then:
\[
|\eta_{i,j}| \leq 2C \left( \frac{n - 1}{n} \right ).
\]
\end{lemma}

\begin{proof}
Starting with:
\begin{align*}
\eta_{i,j} &= c_{i,j} - \frac{1}{n} \sum_{i'=1}^{n} c_{i',j} \\
&= c_{i,j} - \left( \frac{1}{n} c_{i,j} + \frac{1}{n} \sum_{i' \ne i}^{n} c_{i',j} \right ) \\
&= \left( 1 - \frac{1}{n} \right ) c_{i,j} - \frac{1}{n} \sum_{i' \ne i}^{n} c_{i',j}.
\end{align*}
Taking absolute values and using triangle inequality:
\begin{align*}
|\eta_{i,j}| &\leq \left( 1 - \frac{1}{n} \right ) |c_{i,j}| + \frac{1}{n} \sum_{i' \ne i}^{n} |c_{i',j}| \\
&\leq \left( \frac{n - 1}{n} \right ) C + \left( \frac{n - 1}{n} \right ) C \\
&= 2C \left( \frac{n - 1}{n} \right ).
\end{align*}
\end{proof}

\begin{lemma}[Total Error Bound]
The total error per parameter \( \varepsilon_{i,l} \) is given by:
\[
\varepsilon_{i,l} = w_{\text{FedAvg},l} - w_{i,l}^{\text{global}} = \delta_{i,l} - \eta_{i,j},
\]
and is bounded by:
\[
|\varepsilon_{i,l}| \leq B + 2C \left( \frac{n - 1}{n} \right ).
\]
\end{lemma}

\begin{proof}
From the definitions:
\begin{align*}
   \varepsilon_{i,l} &= w_{\text{FedAvg},l} - w_{i,l}^{\text{global}} \\
   &= (w_{\text{FedAvg},l} - w_{i,l}) - (c_{i,j} - c_{\text{avg},j}) \\
   &= \delta_{i,l} - \eta_{i,j}. 
\end{align*}
Taking absolute values:
\[
|\varepsilon_{i,l}| \leq |\delta_{i,l}| + |\eta_{i,j}| \leq B + 2C \left( \frac{n - 1}{n} \right).
\]
\end{proof}

\begin{corollary}[Error Decreases with Increasing \( No_c \)]
As the number of clusters \( No_c \) increases, the total error \( |\varepsilon_{i,l}| \) decreases.
\end{corollary}

\begin{proof}
As \( No_c \) increases:
\begin{enumerate}
    \item The number of parameters per cluster \( |S_{i,j}| \) decreases.
    %\item Variance within clusters decreases.
    \item Centroids represent fewer parameters.
    \item Centroid differences decrease: \( |c_{i,j} - c_{\text{avg},j}| \) decreases.
    \item Approximation error decreases: Overall, \( |\varepsilon_{i,l}| \) decreases with increasing \( No_c \).
\end{enumerate}
Since both \( B \) and \( C \) decrease with increasing \( No_c \), the total error bound \( |\varepsilon_{i,l}| \) decreases.
\end{proof}

\subsection*{Proof of Theorem}

\begin{proof}[Proof of Theorem~\ref{theo:centroids}]
From Lemma 5, the total error per parameter is:
\[
|\varepsilon_{i,l}| \leq B + 2C \left( \frac{n - 1}{n} \right ).
\]
As \( No_c \to r \):
\begin{itemize}
    \item Each cluster contains a single parameter (\( |S_{i,j}| = 1 \)).
    \item Local centroids become the actual parameters (\( c_{i,j} = w_{i,l} \)).
    \item Global centroids become the averages of individual parameters (\( c_{\text{avg},j} = w_{\text{FedAvg},l} \)).
\end{itemize}
Therefore:
\[
\delta_{i,l} \to 0, \quad \eta_{i,j} \to 0, \quad \text{and} \quad |\varepsilon_{i,l}| \to 0.
\]
Thus:
\[
\lim_{No_c \to r} w_{i,l}^{\text{global}} = w_{\text{FedAvg},l}, \quad \forall i, l.
\]
This proves that the clients' updated models converge to the FedAvg global model as \( No_c \to r \).
\end{proof}

\begin{corollary}[Exact Equality at \( No_c = r \)]
When \( No_c = r \), the clients' updated models are exactly equal to the FedAvg global model:
\[
w_i^{\text{global}} = w_{\text{FedAvg}}, \quad \forall i.
\]
\end{corollary}

\begin{proof}
At \( No_c = r \):
\begin{itemize}
    \item Each cluster contains exactly one parameter (\( |S_{i,j}| = 1 \)).
    \item Local centroids: \( c_{i,j} = w_{i,l} \).
    \item Global centroids: \( c_{\text{avg},j} = \frac{1}{n} \sum_{i=1}^{n} w_{i,l} = w_{\text{FedAvg},l} \).
\end{itemize}
Parameter adjustment:
\[
w_{i,l}^{\text{global}} = w_{i,l} + (w_{\text{FedAvg},l} - w_{i,l}) = w_{\text{FedAvg},l}.
\]
Therefore, \( w_i^{\text{global}} = w_{\text{FedAvg}} \) for all clients.
\end{proof}

\subsection*{Conclusion}

The presented theorems and lemmas formally establish that as the number of clusters \( No_c \) increases to the number of model parameters \( r \), the clients' updated models converge to the FedAvg global model. The error analysis quantifies the approximation error introduced by clustering and shows that this error diminishes with increasing \( No_c \), approaching zero when \( No_c = r \).

\end{document}